\pdfoutput=1
\documentclass[runningheads,envcountsect,envcountsame]{llncs}

\PassOptionsToPackage{dvipsnames,table}{xcolor}
\usepackage{eccv}
\usepackage{eccvabbrv}

\usepackage{xspace}
\newcommand{\ours}{\textbf{MoVA}\xspace}

\usepackage{booktabs}
\usepackage{multirow}

\newcommand{\Frst}[1]{\textbf{#1}}
\newcommand{\Scnd}[1]{\underline{#1}}
\definecolor{aliceblue}{rgb}{0.94, 0.97, 1.0}

\definecolor{pastelgreen}{RGB}{188, 212, 187}
\definecolor{darkgreen}{RGB}{0,100,0}
\definecolor{softblue}{RGB}{70, 136, 207}        
\definecolor{dustyrose}{RGB}{220, 140, 156}      
\definecolor{sageleaf}{RGB}{136, 176, 75}        
\definecolor{goldenyellow}{RGB}{255, 198, 109}   
\definecolor{lavender}{RGB}{150, 123, 182}       
\definecolor{seafoam}{RGB}{98, 182, 177}         
\definecolor{terracotta}{RGB}{226, 114, 91}      
\definecolor{periwinkle}{RGB}{143, 161, 215}     
\definecolor{mintgreen}{RGB}{152, 203, 178}      
\definecolor{mauve}{RGB}{224, 176, 187}          

\definecolor{reviewerRone}{RGB}{0, 158, 115}     
\definecolor{reviewerRtwo}{RGB}{0, 114, 178}     
\definecolor{reviewerRthree}{RGB}{156, 32, 7} 

\usepackage{array}
\usepackage{rotating}
\usepackage{float}
\usepackage{lipsum}

\makeatletter
\@namedef{ver@everyshi.sty}{}
\makeatother
\usepackage{tikz}
\usetikzlibrary{bayesnet,arrows,backgrounds}

\usepackage{tcolorbox}

\usepackage{graphicx}
\usepackage{subcaption}
\renewcommand{\paragraph}[1]{\noindent\textbf{#1}}

\newcommand{\stkout}[1]{\ifmmode\text{\sout{\ensuremath{#1}}}\else\sout{#1}\fi}
\usepackage{amsmath,amsfonts,bm}

\providecommand{\1}{\mathbf{1}}

\providecommand{\mm}{\mathbf{m}}

\providecommand{\tt}{\mathbf{t}}

\providecommand{\zz}{\mathbf{z}}

\providecommand{\cB}{\mathcal{B}}

\providecommand{\cK}{\mathcal{K}}

\providecommand{\cM}{\mathcal{M}}

\providecommand{\cS}{\mathcal{S}}
\providecommand{\cT}{\mathcal{T}}

\providecommand{\cV}{\mathcal{V}}
\providecommand{\cX}{\mathcal{X}}

\providecommand{\cZ}{\mathcal{Z}}

\providecommand{\R}{\mathbb{R}} 

\providecommand{\mT}{\mathbf{T}}

\providecommand{\mV}{\mathbf{V}}

\usepackage[noend]{algpseudocode}

\errorcontextlines\maxdimen

\makeatletter
    \newcommand*{\algrule}[1][\algorithmicindent]{\makebox[#1][l]{\hspace*{.5em}\thealgruleextra\vrule height \thealgruleheight depth \thealgruledepth}}%
\newcommand*{\thealgruleextra}{}
\newcommand*{\thealgruleheight}{.75\baselineskip}
\newcommand*{\thealgruledepth}{.25\baselineskip}

\newcount\ALG@printindent@tempcnta
\def\ALG@printindent{%
    \ifnum \theALG@nested>0
        \ifx\ALG@text\ALG@x@notext
        \else
            \unskip
            \addvspace{-1pt}
            \ALG@printindent@tempcnta=1
            \loop
                \algrule[\csname ALG@ind@\the\ALG@printindent@tempcnta\endcsname]%
                \advance \ALG@printindent@tempcnta 1
            \ifnum \ALG@printindent@tempcnta<\numexpr\theALG@nested+1\relax
            \repeat
        \fi
    \fi
    }%
\usepackage{etoolbox}
\patchcmd{\ALG@doentity}{\noindent\hskip\ALG@tlm}{\ALG@printindent}{}{\errmessage{failed to patch}}
\makeatother

\newbox\statebox
\newcommand{\myState}[1]{%
    \setbox\statebox=\vbox{#1}%
    \edef\thealgruleheight{\dimexpr \the\ht\statebox+1pt\relax}%
    \edef\thealgruledepth{\dimexpr \the\dp\statebox+1pt\relax}%
    \ifdim\thealgruleheight<.75\baselineskip
        \def\thealgruleheight{\dimexpr .75\baselineskip+1pt\relax}%
    \fi
    \ifdim\thealgruledepth<.25\baselineskip
        \def\thealgruledepth{\dimexpr .25\baselineskip+1pt\relax}%
    \fi
    \State #1%
    \def\thealgruleheight{\dimexpr .75\baselineskip+1pt\relax}%
    \def\thealgruledepth{\dimexpr .25\baselineskip+1pt\relax}%
}

\usepackage{amsmath,amssymb}

\spnewtheorem{assumption}[theorem]{Assumption}{\bfseries}{\itshape}
\spnewtheorem{condition}[theorem]{Condition}{\bfseries}{\itshape}

\usepackage{enumitem}
\usepackage{soul}
\usepackage{multirow}
\usepackage{algorithm}
\usepackage{caption}
\usepackage{wrapfig}

\providecommand{\rtt}{r_{\mathrm{T}}}
\providecommand{\rii}{r_{\mathrm{I}}}

\renewcommand{\hm}{\hat{\mm}}
\providecommand{\hgV}{\hat{g}^{\mathrm{V}}}
\providecommand{\hgT}{\hat{g}^{\mathrm{T}}}
\providecommand{\hzV}{\hat{\zz}^{\mathrm{V}}}
\providecommand{\hzT}{\hat{\zz}^{\mathrm{T}}}

\providecommand{\Ltv}{\mathcal{L}_{\mathrm{align}}^{\mathrm{t} \to \mathrm{v}} }
\providecommand{\Lvt}{\mathcal{L}_{\mathrm{align}}^{\mathrm{v} \to \mathrm{t}}}
\providecommand{\Lsfdm}{\mathcal{L}_{\mathrm{sfdm}} }
\providecommand{\Ldfsm}{\mathcal{L}_{\mathrm{dfsm}} }
\providecommand{\Lr}{\mathcal{L}_{\mathrm{g}} }
\providecommand{\Lsp}{\mathcal{L}_{\mathrm{s}} }

\providecommand{\lamg}{\lambda_{\mathrm{g}} }
\providecommand{\lamtv}{\lambda_{\mathrm{tv}}}
\providecommand{\lamvt}{\lambda_{\mathrm{vt}}}
\providecommand{\lams}{\lambda_{\mathrm{s}} }

\providecommand{\vV}{\mathbf{V}} 
\providecommand{\vX}{\mathbf{X}} 
\providecommand{\vT}{\mathbf{T}} 

\providecommand{\zT}{\zz^{\mathrm{T}}} 
\providecommand{\zV}{\zz^{\mathrm{V}}} 

\providecommand{\mT}{\mm^{\mathrm{T}}} 
\providecommand{\mV}{\mm^{\mathrm{V}}} 

\providecommand{\veT}{\bm\epsilon^{\mathrm{T}}} 
\providecommand{\veV}{\bm\epsilon^{\mathrm{V}}} 

\providecommand{\gT}{g^{\mathrm{T}}} 
\providecommand{\gV}{g^{\mathrm{V}}} 

\IfFileExists{axessibility.sty}{\usepackage[accsupp]{axessibility}}{}
\usepackage{hyperref}
\IfFileExists{orcidlink.sty}{\usepackage{orcidlink}}{}

\title{MoVA: Learning Asymmetric Dual Projections for Modular Long Video-Text Alignment}
\titlerunning{MoVA: Modular Long Video--Text Alignment}

\author{Peiyuan~Zhu\inst{1} \and
Shaoan~Xie\inst{1,2} \and
Zijian~Li\inst{1,2} \and
Yifan~Shen\inst{1} \and
Namrata~Deka\inst{2} \and
Harsh~Shrivastava\inst{1} \and
Guangyi~Chen\inst{1,2} \and
Kun~Zhang\inst{1,2}}
\authorrunning{P.~Zhu et al.}
\institute{Mohamed bin Zayed University of Artificial Intelligence, Abu Dhabi, UAE\\
\email{\{peiyuan.zhu,zijian.li,yifan.shen\}@mbzuai.ac.ae}\\
\email{\{harsh.shrivastava,guangyi.chen\}@mbzuai.ac.ae}
\and
Carnegie Mellon University, Pittsburgh, PA, USA\\
\email{shaoan@cmu.edu, ndeka@cs.cmu.edu, kunz1@cmu.edu}}

\begin{document}
\maketitle
\begin{abstract}
Contrastive pre-training has propelled video-text alignment, yet models often inherit the critical limitations of their image-text predecessors like CLIP, resulting in entangled representations. These challenges are severely exacerbated by two fundamental properties in the video domain: Temporal Misalignment, where textual descriptions often correlate only to specific, constrained temporal windows, leaving other frames text-irrelevant; and Semantic Asymmetry, which dictates a sparse, bidirectional, and non-equivalent relevance between frame-level visual details and caption-level concepts. This failure persists whether captions are short and temporally disjoint, creating ambiguity, or long and detailed, fostering entanglement between static objects and their temporal evolution. In this paper, we establish theoretical conditions that enable flexible alignment between video and text representations across the temporal dimension and at varying levels of granularity. Building on these theoretical insights, we introduce MoVA—Modular Long Video–Text Alignment—which learns dual asymmetric projections: a text-side projection that adaptively selects frame-aware subspaces of the caption, and a video-side projection that disentangles text-relevant visual concepts.  Our framework ensures that the model can preserve global cross-modal semantics while disentangling evolving, frame-specific concepts and scale naturally to long captions and videos. Empirical evaluations show that MoVA outperforms existing methods in multiple video-text alignment tasks, demonstrating the effectiveness of our method.

\keywords{Video-Text Retrieval \and Multimodality \and Representation Learning}

\end{abstract}

\section{Introduction}
\label{sec:intro}
Developing generalizable video–text representations remains an important problem in modern computer vision. Driven by rapid advances in vision–language pre-training (VLP), large-scale cross-modal representation learning—coupled with explicit video–text alignment objectives—yields aligned embeddings that transfer broadly across a diverse spectrum of downstream tasks, including video-text retrieval~\cite{wu2021hanet,Wang2023AlignAT,Wang2021T2VLADGS,Chen2023TaggingBA,yang2024dgl}, video captioning~\cite{shi2023learning,jiang2025text}, and action recognition~\cite{zhang2024enhanced,chen2024align}.

\begin{figure}[t]
  \centering
   \includegraphics[width=0.95\linewidth]{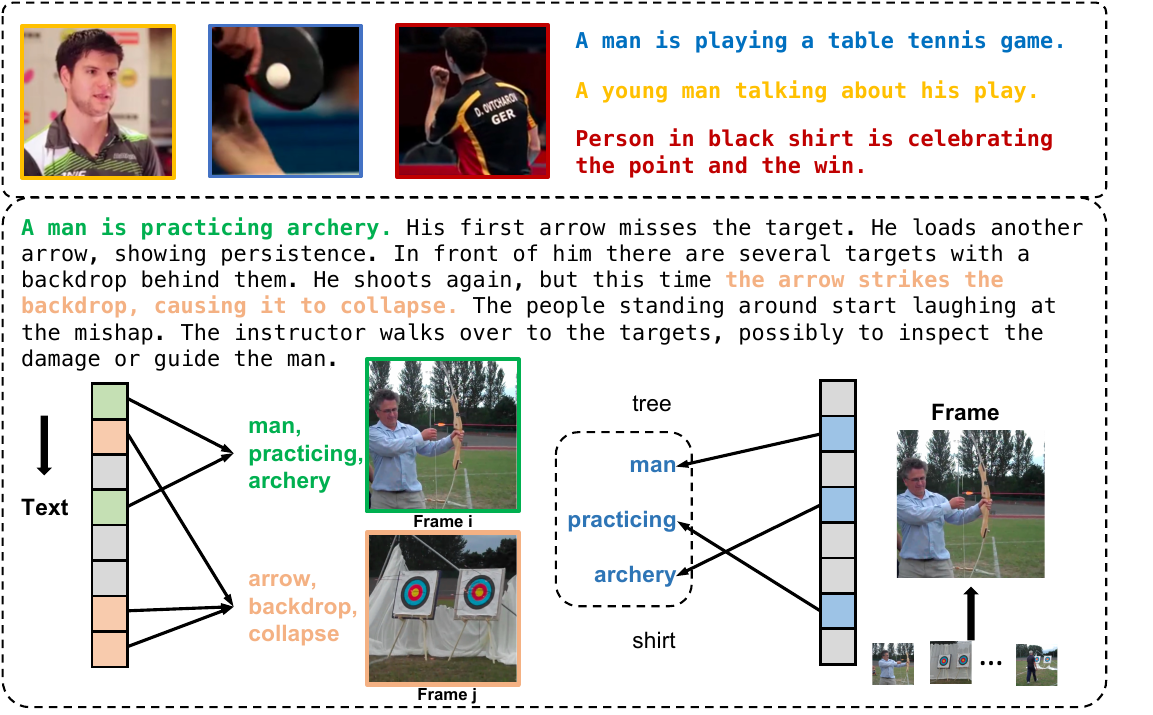}

   \caption{Illustration of two key challenges for video–language alignment.
(1) Temporal Misalignment: Video captions are inherently temporally misaligned with the underlying visual content. 
Multiple paired captions emphasize different moments along the video timeline—the first frame focuses on the man talking, while the second centers on the table tennis game. A single caption may describe an action occurring within only a short temporal window, leaving other frames potentially text-irrelevant. 
(2) Semantic Asymmetry: Regardless of caption length, only a sparse subset of the caption is relevant to any given frame. Text induces selective relevance, influencing each frame’s attention toward different textual components, whereas each frame preserves richer yet underdetermined information toward the correlated text, forming an inherent bidirectional asymmetry between the two modalities.  }
   \label{fig:intro-challenges}
\end{figure}

The success of CLIP-style contrastive learning~\cite{radford2021learning,oord2018representation,chen2020simple} motivated many video–text systems to reuse image–text encoders and perform lightweight post-pretraining.
However, CLIP itself can suffer from information misalignment in many image–text datasets and tends to learn entangled representations~\cite{fan2023improving,lai2024veclip,materzynska2022disentangling,thrush2022winoground,yuksekgonul2023when,lewis2024does}: different captions for the same image may highlight different concepts, and a single caption may involve multiple concepts. SmartCLIP~\cite{Xie2025SmartCLIPMV} addresses this issue with adaptive masking and a modular contrastive objective at the image-level, encouraging disentangled representations of images and texts.

When extending to video, additional temporal misalignment arises because a video is a sequence of frames, risking the loss of salient frame-level semantics and misinterpretation of temporal dynamics. On the one hand, short captions for a single video may describe disjoint temporal segments, leaving the model uncertain about which parts of the text to attend to for a given frame. On the other hand, directly aligning long captions with videos can preserve entangled details, preventing the learning of disentangled, atomic concepts at both the object and temporal levels, and ultimately limiting generalization on video tasks that demand fine-grained understanding across time and entities. We further characterize this as semantic asymmetry, shown in Figure~\ref{fig:intro-challenges}.

To mitigate the impact of the above factors, we formulate video–text alignment as a latent-variable identification problem over temporally indexed frames and textual spans. We establish conditions under which frame–text correspondences are recoverable from weak clip–caption supervision, enabling our framework to preserve information as it evolves over time and to disentangle object-centric from motion-centric factors.
Building on these theoretical insights, we introduce MoVA—a Modular Video–Text Alignment framework for CLIP-style video-text models.
Operationally, MoVA decomposes the text representation into concept modules and learns a frame-aware mask selector that adaptively activates the relevant textual components per frame. The modular contrastive objective is computed at the frame level and aggregated across clips with lightweight temporal-consistency and coverage regularization.
We empirically demonstrate that MoVA achieves competitive results across a range of downstream tasks without extra post-pretraining data, showcasing its effectiveness in addressing alignment challenges.

Our main contributions are summarized as follows.
\begin{itemize}[leftmargin=*,topsep=0pt]
  \item Theory for video--text misalignment and disentanglement.
  We formally characterize the challenges of information misalignment and semantic asymmetry in video--text alignment and cast the problem in a latent--variable framework and derive identification conditions that guarantee recovery of frame--text correspondences and concept-level factors.

  \item We propose dual asymmetric projections with modular contrastive learning procedures for video-text alignment, promoting disentangled, compositional representations. 

  \item We conduct extensive experiments across diverse tasks, including long and short video-text retrieval, text-to-video generation, and concept visualization. MoVA consistently achieves competitive results, demonstrating its efficacy and validating our theoretical contributions.

\end{itemize}

\section{Related Work}
\label{sec:formatting}

\subsection{Video-Text Alignment}
Modern video-text alignment research is heavily influenced by the wave of contrastive learning in the image-text domain.
CLIP demonstrated the immense potential of natural language supervision for training robust visual models. This image-text pre-training paradigm was rapidly adapted to the video domain in an end-to-end manner, outperforming traditional non-CLIP methods~\cite{patrick2020support,croitoru2021teachtext} and
serving as effective video learners~\cite{rasheed2023fine}.
The pioneering work CLIP4Clip~\cite{luo2022clip4clip} investigates various similarity calculation mechanisms and the effect of post-pretraining on video-language datasets, which spurs a significant volume of subsequent research into video-text alignment.
Regarding encoder extraction~\cite{xue2022clip,deng2023prompt,wang2022object} focus on building more precise and comprehensive encoders for video and text modalities.
For interaction modeling between text and video modalities, works like~\cite{ma2022x,wang2023unified,tian2024towards,yang2024dgl} learn various combinations of coarse- and fine-grained, as well as global and local representations. For instance, DGL~\cite{yang2024dgl} utilizes a shared latent space to generate dynamic local prompts and employs a global-local attention mechanism.
For temporal sequence alignment, DTW (Dynamic Time  
Warping)~\cite{sakoe2003dynamic} utilizes a symmetric distance definition and imposes strong temporal constraints to align the two sequences. VT-TWINS~\cite{Ko2022VideoTextRL} aligns noisy video-text pairs via a differentiable DTW that handles weak correlations using local neighborhood smoothing. Other works, such as~\cite{wang2022disentangled,Jin2023TextVideoRW}, have focused on learning disentangled representations for video-text alignment.
More recently, many approaches train larger models—even foundation models—on substantially richer video corpora, extending video–text alignment far beyond retrieval to domains such as video understanding~\cite{chen2023vast,wang2023internvid,wang2024videoclip,wang2024internvideo2} and video generation~\cite{bai2025qwen2,wan2025}. 
Long video–caption pairs are also becoming crucial: datasets like LVD-2M~\cite{xiong2024lvd}, VideoUFO~\cite{wang2025videoufo}, and UltraVideo~\cite{xue2025ultravideo} offer longer, more detailed, and richer video-caption pairs. Long captions are critical as they tend to retain entangled details~\cite{zhang2024long,Xie2025SmartCLIPMV}, and original videos with sufficient change and clear content help expand downstream applications of video–text alignment beyond the confines of retrieval.
Our method focuses on addressing the temporal misalignment and semantic asymmetry inherent in video-text pairs through mask modeling to learn temporally interpretable concepts for long video–text alignment.

\subsection{Latent Variable Identification}
Latent variable identification is the process of statistically inferring and uniquely recovering unobservable, hidden variables or constructs from a set of observable, measured data. 
Identifying the latent structure that mediates video–text correspondence is a principled route to robust alignment and generalization. A large body of work shows that, under suitable auxiliary signals or structural conditions, semantic factors become recoverable from observations. 
Recent research further demonstrates that latent causal variables become identifiable once additional structure is imposed on the learning problem. Temporal regularities—modeling how latent factors evolve causally over time or exhibit temporal sparsity—supply such structure and can break indeterminacies \cite{yao2021learning,yao2022learning,klindt2020towards,hyvarinen2017nonlinear}. Complementary advances constrain the generative mechanism itself: sparsity and related structural priors shrink equivalence classes and enable recovery of the underlying factors \cite{xu2024sparsity,zheng2022identifiability,zheng2023generalizing}. Cross-modal alignment also provides anchors: paired or multi-view observations allow contrastive or grouping-based objectives to tie representations across modalities and render the latent structure identifiable even under partial observability \cite{yao2023multi,hyvarinen2017nonlinear,morioka2023connectivity,morioka2023causal,daunhawer2023identifiability,sun2025causal,gresele2020incomplete,chen2025causalverse}. Under appropriate conditions, the identification of disentangled and manipulable latents enables controllable image and video generation \cite{xielearning,shen2025controllable,kim2025subject}.
For video–text alignment, the objective is to learn high-level semantics from low-level observational video–text pairs. These paired data can be treated as multi-view observations that enable identification of information shared across views. Previous work~\cite{daunhawer2023identifiability,von2021self} adopts flexible assumptions on the underlying distribution to identify blocks of latent variables directly shared by two views induced by data augmentations, and~\cite{yao2023multi} extends this to the multi-view setting. However, these approaches face persistent grouping ambiguity (e.g., whether captions from different videos belong to the same group). SmartCLIP~\cite{Xie2025SmartCLIPMV} approaches alignment from the image–text side, where the grouping of multiple captions—short or long—is coherently determined by preserved cross-modal information, allowing the image to unilaterally select aspects of the textual description. In contrast, video–text pairs are inherently bidirectional: the text abstracts transformations across frames and can even determine which frames are unimportant, rendering a simple one-to-one image-to-text mapping untenable. In our theoretical analysis, we show that by properly leveraging the data-generating process, we can obtain the desired identification results for video–text alignment.

\section{Problem Formulation}
\label{sec:formulation}
In this section, we formalize the data-generating process underlying video–text alignment as the basis for subsequent theoretical analysis.

\begin{figure}[t]
  \centering
   \includegraphics[width=0.7\linewidth]{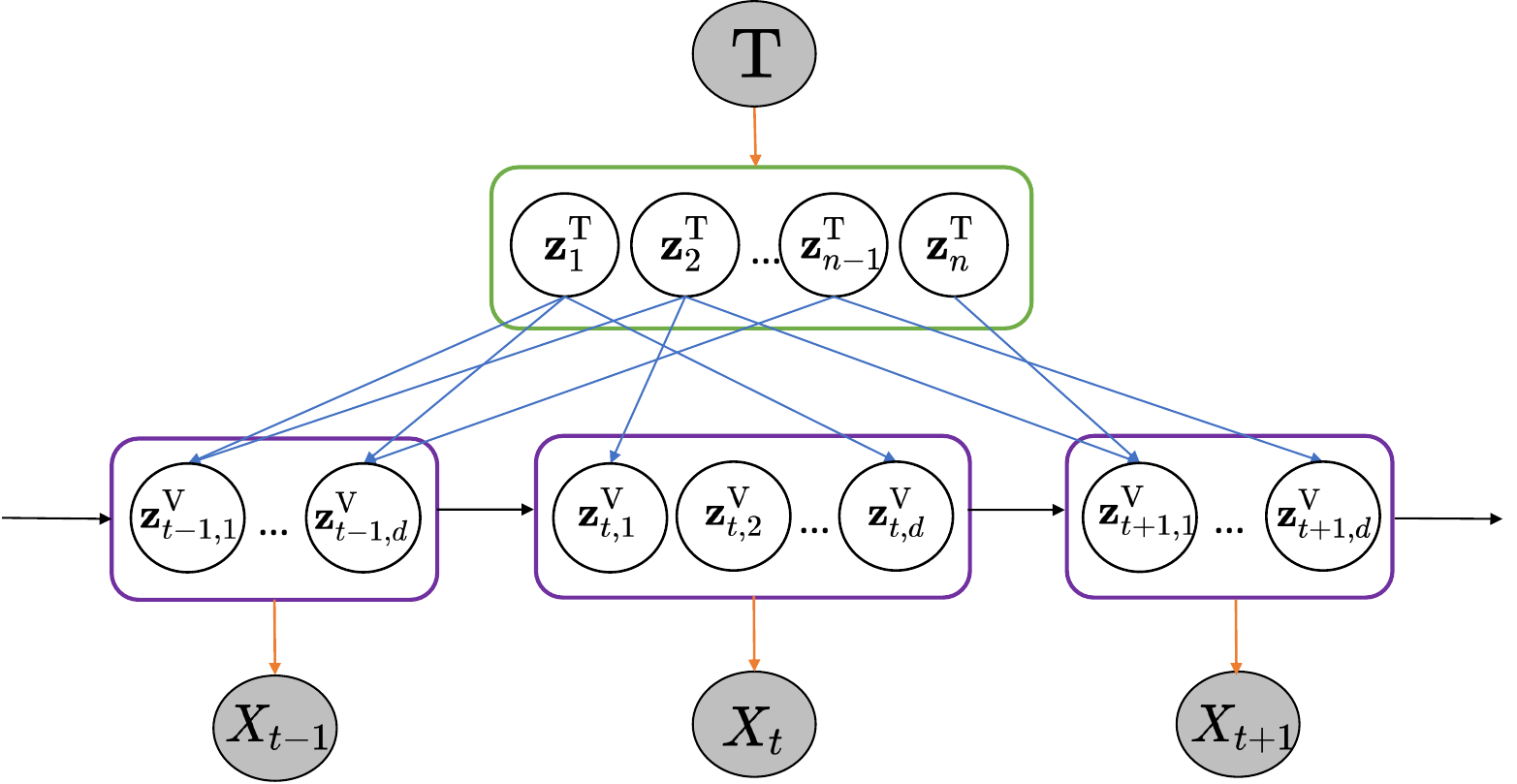}

   \caption{The \textbf{data-generating process}. The video-text pair originates from its corresponding representation pair.  
   The mapping from the text representation $\mathbf{z}^{\mathrm{T}}$ to the sequence of frame representations $\mathbf{z}^{\mathrm{V}}$ over time is sparse. The subset of text representations associated with a particular frame representation contains only partial information about the vision representation. }
   \label{fig:data-generating-process}
\end{figure}

\paragraph{Notations.}
We indicate the dimensionality of a vector with $d(\cdot)$ and index a subset of its components by $[\zz]_{\cB}$ for an index set $\cB$.
For any mask vector $\mm$, we write the support (nonzero indices) as $\cB(\mm):=\{i\in[d(\mm)]:[\mm]_i\neq 0\}$. We denote the element-wise (Hadamard) product by $\odot$.

\paragraph{Data-generating process.}
The data-generating process is shown in Figure~\ref{fig:data-generating-process} and specified in (\ref{eq:vt_scm})-(\ref{eq:vt_constraint}).
A video is $\vV:=(\vX_1,\ldots,\vX_T)$ with frames $\vX_t\in\cX\subset\R^{d(\vX_t)}$ and a paired caption $\vT\in\cT\subset\R^{d(\vT)}$.
Both modalities live in a \emph{shared} latent space $\cZ\subset\R^{d(\zz)}$: a text latent $\zT\in\cZ$ and per-frame latents $\zV_t\in\cZ$ for $t=1,\ldots,T$.
Since each word can be represented by a continuous word embedding vector in practice~\cite{bengio2003neural}, we model the text caption as continuous variables.
We assume that each video-text pair $(\vV,\vT)$ originates from semantic latents together with modality-specific nuisance variables through modality-specific generators: the video generator $\gV$ and the text generator $\gT$. Concretely, frames are generated by $\gV:(\zV_t,\veV_t)\mapsto \vX_t$ and the caption is generated by $\gT:(\zT,\veT)\mapsto \vT$.
To capture temporal locality and semantic asymmetry, we introduce \emph{dual sparse asymmetric} projections:
 \begin{enumerate}[label=\textit{\roman*},leftmargin=1em, topsep=1pt, partopsep=0pt, itemsep=-0.0em]
    \setlength\itemsep{-0.0em}
        \item \label{asmp:invertibility} Text$\to$frame mask $\mm_t^{\mathrm{T}}\in\cM^{\mathrm T}\subset\{0,1\}^{d(\zz)}$, which selects a subset of text concepts relevant to frame $t$ (global text $\to$ local frame).
        \item \label{asmp:full_support} Frame$\to$text mask $\mm_t^{\mathrm{V}}\in\cM^{\mathrm V}\subset\{0,1\}^{d(\zz)}$, which selects the subset of frame-level visual factors accountable to the subset of text concepts (local frame $\to$ local text).
\end{enumerate}
We assume modality-specific generators with nuisance variations $\veV_t$ (e.g., motion blur, illumination) and $\veT$ (e.g., syntax, tense), and encode frame-aware dual selection as a per-frame semantic-consistency constraint, as shown in (\ref{eq:vt_constraint}).
\begin{align}
\vX_t := \gV(\zV_t,\veV_t),\qquad \vT :&= \gT(\zT,\veT),\label{eq:vt_scm}\\
\mm_t^{\mathrm{T}}\odot \zT = \zV_t \odot \mm_t^{\mathrm{V}},\quad t&=1,\ldots,T.
\label{eq:vt_constraint}
\end{align}

\paragraph{Goal.} Our goals can be formalized as follows.
\begin{enumerate}[label=\alph*,leftmargin=1.2em,itemsep=0pt,topsep=2pt]
\item \textit{Preserve global cross-modal semantics over time}: recover the complete caption-level latent $\zT$ and maintain its consistency with the span $\zV_t \mid_{t=1}^T$.
\item \textit{Disentangle frame-specific concepts at multiple granularities}: identify and separate temporally localized factors within $\zV_t$ that correspond to sparse, caption-conditioned subspaces selected by $(\mm_t^{\mathrm{T}},\mm_t^{\mathrm{V}})$, even when such atomic factors are unseen during training.
\end{enumerate}

\paragraph{Examples.}
As illustrated in Figure~\ref{fig:data-generating-process}, the global caption contains two events: ``A man is practicing archery.'' and ``The arrow strikes the backdrop, causing it to collapse.'' 
Consider two frames at different time steps, indexed by $i$ and $j$ ($i<j$ in this case). From the text$\to$frame perspective,
for the frame $\vX_{i}$, the text$\to$frame mask $\mm_{i}^{\mathrm T}$ activates $\{\text{man},\text{practicing},\text{archery}\}$; for the later frame $\vX_{j}$, $\mm_{j}^{\mathrm T}$ activates $\{\text{arrow},\text{backdrop},\text{collapse}\}$. 
From the frame$\to$text perspective, each frame may include additional visual content beyond what is explicitly mentioned (e.g., trees, shirt for $\vX_{i}$). 
Accordingly, $\mm_{i}^{\mathrm V}$ selects the subset of visual factors in $\vX_{i}$ that correspond to the text subset chosen at time $i$ (e.g., archer’s posture, bow, arrow motion rather than background trees), while $\mm_{j}^{\mathrm V}$ prioritizes the arrow–backdrop contact and collapse dynamics at time $j$.

\section{Identification Theory}
\label{sec:theory}
We establish the theoretical guarantees that motivate the modular design in Section~\ref{sec:mova}. We show that a suitable objective recovers the latent semantics shared by video and text up to block-wise equivalence, even though captions provide only weak, temporally sparse supervision.
\begin{definition}[Block-wise Identifiability]
\label{def:block_identifiability}
A latent vector $\mathbf{v}$ is block-wise identifiable if it is related to its estimate $\hat{\mathbf{v}}$ through an invertible map on every block selected by the associated mask $\,\mm$. 
\end{definition}

\paragraph{Learning objective.}
Let $\hzV_t=\hgV(\mathbf{X}_t)$ and $\hzT=\hgT(\mathbf{T})$ denote the encoder outputs defined in Section~\ref{sec:mova}. We estimate frame-wise masks $\hat{\mm}^{\mathrm{V}}_t,\hat{\mm}^{\mathrm{T}}_t\in\mathcal{M}$ by solving
\begin{equation}
\label{eq:contrastive_objective_mova}
\begin{split}
\min_{\hgV,\hgT,\{\hat{\mm}^{\mathrm{V}}_t,\hat{\mm}^{\mathrm{T}}_t\}} \quad &
\sum_t \big(\|\hat{\mm}^{\mathrm{V}}_t\|_{0}+\|\hat{\mm}^{\mathrm{T}}_t\|_{0}\big) \\
\text{s.t.} \quad 
\sum_t &\big\|\hzV_t\odot \hat{\mm}^{\mathrm{V}}_t - \hzT\odot \hat{\mm}^{\mathrm{T}}_t\big\|^2<\eta.
\end{split}
\end{equation}
The constraint captures the limiting form of the modular contrastive loss in Section~\ref{sec:mova}, forcing the dual-mask relation between $\hat{\mm}^{\mathrm{T}}_t\odot \hzT $ and $ \hzV_t \odot \hat{\mm}^{\mathrm{V}}_t$ to be close for every frame. The sparsity objective encourages minimal supports on both sides so that only frame-relevant coordinates remain active.

\begin{condition}[Identification Conditions]
\label{cond:identification_mova}
\mbox{}

\begin{enumerate}[label=(\roman*),leftmargin=2em]
    \item Smoothness \& invertibility. The generators $(\gV,\gT)$ in Section~\ref{sec:mova} are smooth and admit smooth inverses, so no semantic information is lost in $(\mathbf{X}_t,\mathbf{T})$.
    \item Full joint support. Every pair $(\zz,\mm)$ that satisfies the dual-mask constraint in Section~\ref{sec:mova} occurs with positive density, ensuring that concepts are observed under all admissible temporal spans.
    \item Temporal stability. Each video admits segments $\{\mathcal{S}_{k}\}$ such that the true text-side mask is constant on $\mathcal{S}_{k}$ while visual masks may vary as long as the dual-mask constraint holds.
    \item View diversity. Conditioned on $(\zz,\mm)$, the nuisance variables $(\veV_t,\veT)$ vary in a neighborhood, yielding multiple conditionally independent realizations of every semantic block.
\end{enumerate}
\end{condition}

\begin{theorem}[Identification Theorem]
\label{thm:identification_mova}{ }
Assume the data-generating process in Section~\ref{sec:formulation} and let $(\hgV,\hgT,\{\hat{\mm}^{\mathrm{V}}_t,\hat{\mm}^{\mathrm{T}}_t\})$ be an optimum of \eqref{eq:contrastive_objective_mova}. Under Condition~\ref{cond:identification_mova}, the true representation block $[\mathbf{z}]_{\tilde{\mathcal{B}}}$ is block-wise identifiable for any index set $\tilde{\mathcal{B}}$ that can be written as either $\cup_{\mm\in\mathcal{V}}\mathcal{B}(\mm)$ or $\cap_{\mm\in\mathcal{V}}\mathcal{B}(\mm)$ over any subset of masks $\mathcal{V}\subset\mathcal{M}$.
\end{theorem}

\paragraph{Concept preservation.}
Theorem~\ref{thm:identification_mova} ensures that the concept block associated with a temporal span $\mm$ is retained in the learned representation. Hence, even when a caption highlights only a short portion of a long video, MoVA preserves the frame-level semantics selected by $\mm$, preventing unrelated frames from overwriting them.

\paragraph{Concept disentanglement.}
The intersection operation in Theorem~\ref{thm:identification_mova} allows us to recover atomic concepts that recur across temporally disjoint captions. For example, repeated mentions of the same actor or action can be isolated by intersecting the corresponding masks, which explains the compositional behavior observed in Section~\ref{sec:experiments}.

\paragraph{Theoretical contribution.}
Theorem~\ref{thm:identification_mova} extends the multi-view identification frameworks of~\cite{von2021self,yao2023multi} to temporally indexed video--text data without requiring explicit segment labels. Whereas prior work presumes the view group of each caption is known, our objective infers the grouping automatically through the dual-mask constraint, providing the first temporal identification guarantee for CLIP-style video--text models. For the proof of the identification theory, please refer to the supplementary material.

\section{MoVA: Modular Video-Text Alignment}
\label{sec:mova}

\begin{figure*}[t]
  \centering
   \includegraphics[width=1\linewidth]{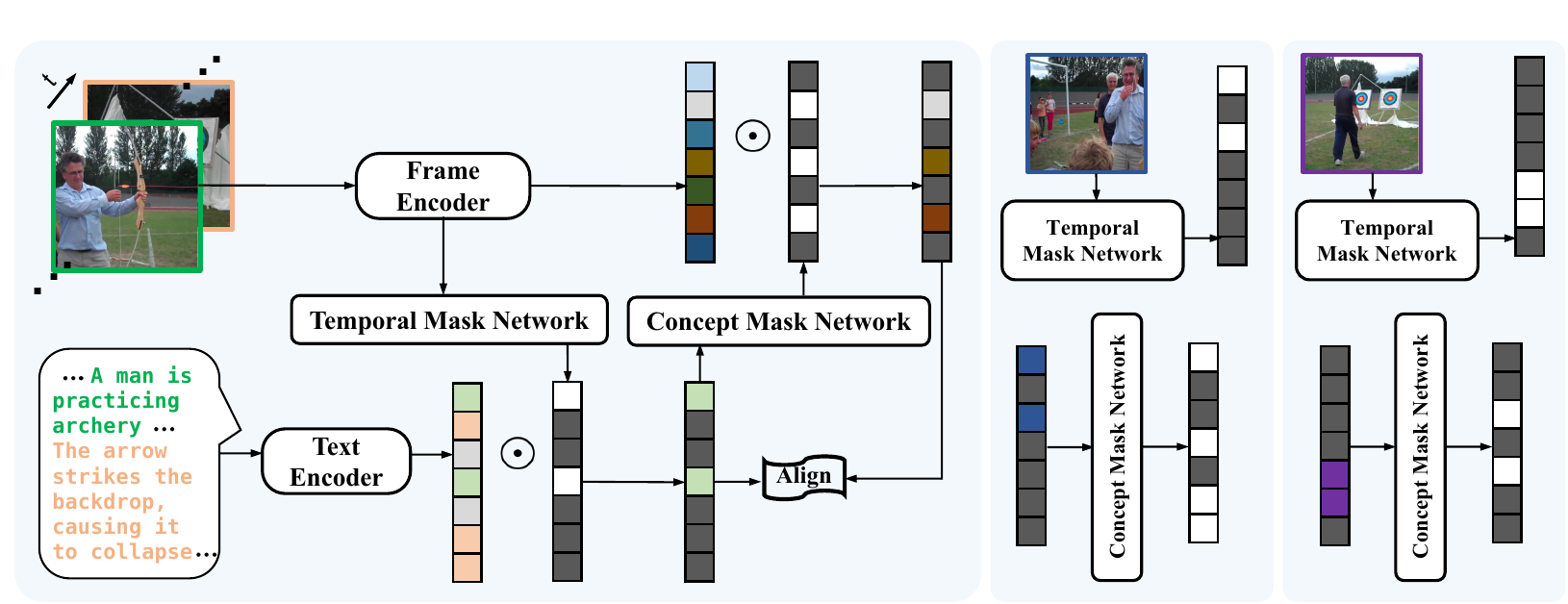}

   \caption{MoVA overview. We integrate dual asymmetric projections: the Temporal Mask Network selects which subset of the global text representation to be used, while the Concept Mask Network selects the parts of the frame representation that align with the selected, correlated text.}
   \label{fig:overview}
\end{figure*}

In this section, we present our empirical approach to modular video-text alignment based on the identification theory in Section~\ref{sec:theory}, detailing the asymmetric dual projections, model architecture and training objectives.

\paragraph{Text-to-Video Projection Learning.}
From the global text representation to each local frame, the text-to-video projection aims to account for frames’ differing focal points on the text, pairing every frame with the portion of the text most relevant to it.
$\mm^{\mathrm{T}}(\cdot)$ denotes the masking projection of $\vX_t$. The Temporal Mask Network (TMN) consists of a two-layer Transformer block. It takes the frame and its learnable positional embedding as the query, allows every token to interact with that frame and with all other tokens in the sequence, and outputs a binary vector $\mm^{\mathrm{T}}(\hzV_t)$ via a straight-through estimator~\cite{bengio2013estimating}, i.e., the subset of text most associated with that frame.
Starting from each frame, we (i) pull the frame’s visual embedding $\hzV_{t}$ closer to its corresponding masked text embedding $\hzT_{t} \odot \hm_t^{\mathrm{T}}$ and (ii) enforce that for the current frame $t$, its similarity to its own masked text exceeds its similarity to the masked texts of other frames (e.g., $t-1$). Frames farther away in the current video (beyond the temporal window) and frames from other video samples are regarded as harder examples. We define the contrastive frame-level loss as follows:
\begin{equation}
\label{eq:frame_loss}
\begin{aligned}
    \ell_{\text{ctrf}}(t&) = \left[ 1 - \langle \hzV_{t}, \hzT_{t} \odot \hm_t^{\mathrm{T}}\rangle \right] + \\
    &\left[ \max(0, \Delta - \langle \hzV_{t}, \hzT_{t}\odot \hm_t^{\mathrm{T}}  \rangle + \langle \hzV_{t}, \hzT_{k}\odot \hm_t^{\mathrm{T}} \rangle) \right],
\end{aligned}
\end{equation}
where $\langle\cdot,\cdot\rangle$ denotes the cosine similarity; $\Delta$ is the contrastive margin; and $\hzT_{k}$ denotes the $k$-th negative sample. For brevity, we omit the explicit sample index $i$ in (\ref{eq:frame_loss}). The aggregate loss is given by:
\begin{equation}
    \Ltv=\frac{1}{L}\sum_{i=1}^N\sum_{t=1}^{L_i}\ell_{\mathrm{ctrf}}(t),
\end{equation}
 where $L_i$ denotes the valid frames per sample, $L=\sum_{i=1}^N\sum_{t=1}^{L_i}$ This is consistent with our goal of addressing temporal misalignment and enforces alignment between the global text and the local frames.

\paragraph{Video-to-Text Projection Learning.}
Each frame in the video preserves, in its entirety, the cross-modal semantic information contained in the temporally combined text subset obtained from text-to-video projection learning. We shift the control from text to frames: each frame selects the portion most related to its own visual information and aligns it with the correlated text subset. Following \cite{Xie2025SmartCLIPMV}, we introduce temporal modeling and extend it to video–text alignment.
We build the Concept Mask Network (CMN), \(\hm_t^{\mathrm{V}}(\cdot)\), which learns a masking projection of text subset \(\hzT_{t}\odot \hm_t^{\mathrm{T}}\), which we denote as $\hat{\mathbf{s}}^{\mathrm{T}}$ for brevity. It consists of a single-layer Transformer block and an attention-pooling layer that adaptively down-samples the output to match the dimensionality of the CLIP representation.

To disentangle frame-specific factors and ensure identifiability, we learn the projection with two modular contrastive terms. 
Same-Frame Different-Mask (sfdm) fixes a frame representation $\hzV_t$ (encoder output, normalized) and applies caption-conditioned visual-dimension masks produced from different captions; letting $\hm_{j,t}$ denote the selector induced by caption $j$ fot $t$-th frame and using temperature $\tau$,
\begin{equation}
\label{eq:sfdm}
\mathcal{L}_{\mathrm{sfdm}}
= -\frac{1}{L}\sum_{i=1}^{N}\sum_{t=1}^{L_i}
\log
\frac{\exp\!\Big(\langle \hm_{i,t}^{\mathrm{V}}\!\odot\!\hzV_{i,t},\, \hat{\mathbf{s}}^{\mathrm{T}}_i\rangle/\tau\Big)}
{\sum_{j=1}^{N^\prime}\exp\!\Big(\langle \hm_{j,t}^{\mathrm{V}}\!\odot\!\hzV_{i,t},\, \hat{\mathbf{s}}^{\mathrm{T}}_j\rangle/\tau\Big)} \, .
\end{equation}
Different-Frame Same-Mask (dfsm) contrasts the text subset $\hat{\mathbf{s}}^{\mathrm{T}}$ in the positive  pair with randomly sampled frame representations (from both intra-video and inter-video):
\begin{equation}
\label{eq:dfsm}
\mathcal{L}_{\mathrm{dfsm}}
= -\frac{1}{L}\sum_{i=1}^{N}\sum_{t=1}^{L_i}
\log
\frac{\exp\!\Big(\langle \hm_{i,t}^{\mathrm{V}}\!\odot\!\hzV_{i,t},\, \hat{\mathbf{s}}^{\mathrm{T}}_i\rangle/\tau\Big)}
{\sum_{j=1}^{N^\prime}\exp\!\Big(\langle \hm_{i,t}^{\mathrm{V}}\!\odot\!\hzV_{j,t},\, \hat{\mathbf{s}}^{\mathrm{T}}_i\rangle/\tau\Big)} \, .
\end{equation}
Thus, we have the alignment objective from frames to text:
\begin{equation}
    \Lvt=\mathcal{L}_{\mathrm{sfdm}}+\mathcal{L}_{\mathrm{dfsm}}
\end{equation}

\paragraph{\ours training objective.}
We utilize a global symmetric retrieval loss $\Lr$
following~\cite{luo2022clip4clip} as grounding, where the similarity matrix is given by cosine similarities between global video and text representations. 
We impose sparsity constraints for both $\hm^{\mathrm{T}}$ and $\hm^{\mathrm{V}}$
, encouraging textual and visual concepts to be encoded in a minimal set of latent dimensions, thereby facilitating the disentanglement of distinct concepts. 
The training objective of MoVA is a weighted sum of aforementioned loss terms:
\begin{equation}
    \mathcal{L}=\lamg\Lr+\lamtv\Ltv+\lamvt\Lvt+\lams\Lsp
\end{equation}

\section{Experiments}
\label{sec:experiments}
\begin{table*}[t]
\centering
\footnotesize
{
{
\caption{\textbf{Retrieval performance on ActivityNet.} ``$\uparrow$'' denotes that higher is better. ``$\downarrow$'' denotes that lower is better. \textbf{Bold} and \underline{underlined} values denote the best and second-best results, respectively. $^\dagger$ denotes that the method uses DSL~\cite{Cheng2021ImprovingVR} as post-processing operations.}
\label{activitynet_results}
\resizebox{\textwidth}{!}{%
\begin{tabular}{p{80pt}|ccccc|ccccc}
\toprule[1.5pt]
\multirow{2}{*}{\textbf{Method}} &  \multicolumn{5}{c|}{\textbf{Text $\rightarrow$ Video}} & \multicolumn{5}{c}{\textbf{Video $\rightarrow$ Text}}  \\ 
\cmidrule(rl){2-6}\cmidrule(rl){7-11}
  & {R@1$\uparrow$} & R@5$\uparrow$ & R@10$\uparrow$ & MdR$\downarrow$ & MnR$\downarrow$
& {R@1$\uparrow$} & R@5$\uparrow$ & R@10$\uparrow$ & MdR$\downarrow$ & MnR$\downarrow$    \\ \midrule
CLIP4Clip~\cite{luo2022clip4clip} & 43.8 & 74.9 & 86.6 & 2.0 & \Scnd{6.4} & 42.9 & 75.3 & 86.6 & 2.0 & 6.4 \\
DRL~\cite{wang2022disentangled} & 44.2 & 73.9 & 84.1 & 2.0 & 7.9 & 42.7 & 73.8 & 84.5 & 2.0 & 7.7 \\
X-CLIP~\cite{ma2022x} & 42.9 & 73.7 & 84.7 & 2.0 & 7.4 & 42.2 & 74.6 & 85.5 & 2.0 & 6.9 \\
ProST~\cite{Li2023ProgressiveSP} & 44.5 & 72.5 & 83.6 & 2.0 & 8.5 & 43.7 & 73.9 & 84.7 & 2.0 & 7.0 \\
InternVideo~\cite{wang2023internvid} & 42.2 & 73.1 & 84.7 & 2.0 & 7.7 & 41.6 & 73.8 & 85.0 & 2.0 & 7.1 \\
DiCoSA~\cite{Jin2023TextVideoRW} & 43.7 & 73.8 & 84.4 & 2.0 & 8.0 & 40.6 & 71.8 & 83.8 & 2.0 & 7.7 \\
DGL~\cite{yang2024dgl} & 43.5 & 74.2 & 84.8 & 2.0 & 8.2 & 43.4 & 73.6 & 85.1 & 2.0 & 7.9 \\
VideoCLIP-XL~\cite{wang2024videoclip} & 46.9 & 75.1 & 86.3 & 2.0 & 6.6 & 37.7 & 68.2 & 81.1 & 2.0 & 10.1 \\ 
\midrule
 \rowcolor{aliceblue!60} \textbf{MoVA ~~(Ours)}  & \Scnd{47.8} & \Scnd{77.4} & \Scnd{87.0} & 2.0 & 7.0 & \Scnd{46.7} & \Scnd{76.6} & \Scnd{86.8} & 2.0 & \Scnd{6.3}      \\
 \rowcolor{aliceblue!60} \textbf{MoVA$^\dagger$ (Ours)}  & \Frst{53.6} & \Frst{79.1} & \Frst{88.1} & \Frst{1.0} & \Frst{6.3} & \Frst{54.4} & \Frst{79.9} & \Frst{88.6} & \Frst{1.0} & \Frst{5.7}      \\ \bottomrule[1.5pt]
\end{tabular}
}
}}

\end{table*}

\subsection{Setup}
\paragraph{Datasets.}
MoVA is evaluated on both long and short video--text benchmarks. Classic video-text retrieval datasets ActivityNet~\cite{krishna2017dense}, MSVD~\cite{chen2011collecting} (YouTube2Text), and DiDeMo~\cite{anne2017localizing} are utilized for evaluation, following the standard train/val/test splits~\cite{luo2022clip4clip}. To test scalability to longer descriptions, we adopt the recently released VideoUFO~\cite{wang2025videoufo} and UltraVideo~\cite{xue2025ultravideo}.

\paragraph{Implementation details.}
To facilitate the transfer of sparse mapping information from image-text alignment to video-text alignment, we first train on the image-caption dataset ShareGPT4v~\cite{Chen2023ShareGPT4VIL} following~\cite{Xie2025SmartCLIPMV}, and use the resulting weights to initialize our text encoder, frame encoder, and Concept Mask Network. We adopt a positional encoding capable of handling 248 tokens, overcoming the 77-token limit in the original CLIP.
For mask modeling, we apply sigmoid to restrict the output to the range (0, 1) and employ straight through estimation (STE)~\cite{bengio2013estimating} to binarize the outputs.
The initial learning rate for text encoder and frame encoder is $10^{-7}$, and the initial learning rate for other modules is $10^{-4}$.
Unless stated otherwise, all methods use ViT-B/16 initialization during training with batch size $256$. For VideoCLIP-XL~\cite{wang2024videoclip}, we evaluate the publicly released checkpoint, which was trained with ViT-L/14 initialization.
We set the max token length, max frame length, and number of training epochs to 64, 64, and 20 for ActivityNet, DiDeMo, VideoUFO, and UltraVideo, and to 32, 12, and 3 for MSVD; for VideoCLIP-XL and our method, the maximum token length is 248.
DSL~\cite{Cheng2021ImprovingVR} post-processing is only applied when explicitly noted.

\paragraph{Metrics.}
We use retrieval metrics that capture both precision and ranking stability. Recall at K (R@1/5/10) quantifies whether the correct item appears within the top-K retrieved results, while Median Rank (MdR) and Mean Rank (MnR) diagnose the heavy-tail behavior introduced by long ambiguous descriptions. All metrics are reported for both text$\rightarrow$video and video$\rightarrow$text scenarios following prior work~\cite{luo2022clip4clip,ma2022x}. 

\subsection{Results}

\paragraph{Video-text retrieval.}
In this work, we focus on methods that adapt image-pretrained CLIP models to video without large-scale video post-pretraining corpora. Accordingly, our controlled comparisons emphasize approaches whose initialization and supervision mainly come from image-based CLIP pretraining. Methods that leverage extra post-pretraining video data, e.g., CLIP-ViP trained with HD-VILA-100M~\cite{xue2022clip, xue2022hdvila} and VidLA trained with YT-VidLA-800M~\cite{rizve2024vidla}, are outside the main controlled setting. We include VideoCLIP-XL~\cite{wang2024videoclip} as a reference for long-caption video-text retrieval.
We compare our model with several state-of-the-art works on the video-text retrieval task. 
Our model achieves the best results on all datasets even without post-processing like DSL, as shown in Tables~\ref{activitynet_results} and~\ref{tab:retrieval_results_t2v}. 
We find that for nearly all methods, R@1 in text-to-video retrieval is higher than in video-to-text, which is the opposite of the pattern in image–text retrieval (where image-to-text R@1 typically exceeds text-to-image). This highlights the greater complexity of video–text alignment relative to the near one-to-one mapping in image–text alignment, underscores the guiding role of the global text representation in video–text alignment, and shows the benefit of leveraging it to help disentangle the representations.
The gains are consistent across all caption–video length regimes—from short-caption/short-video (e.g., MSVD) to long-caption/long-video (e.g., VideoUFO)—demonstrating strong bidirectional vision–language correspondence and favorable scaling behavior.

\begin{figure}[t]
  \centering
   \includegraphics[width=1.0\linewidth]{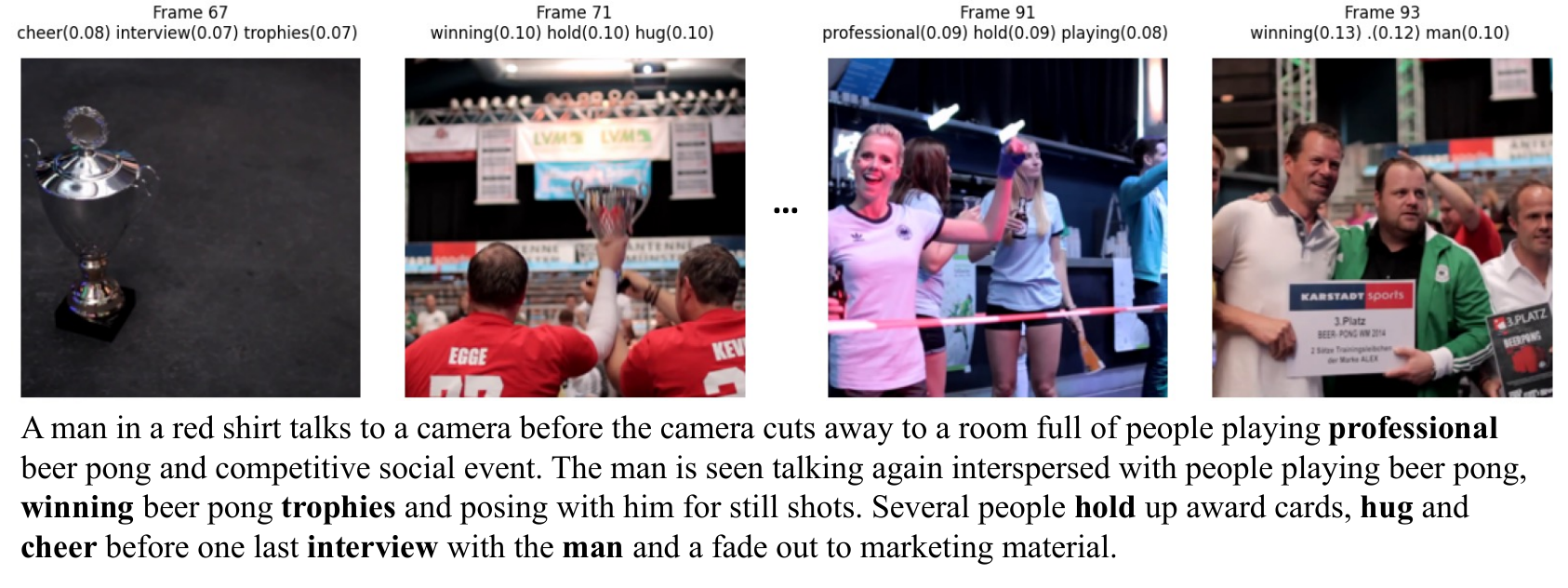}
   \vspace{-2em}
   \caption{A test example on ActivityNet illustrating the most relevant text subsets and weights assigned by the Temporal Mask Network to each frame (top-3 only).  }
   \label{fig:top3words}
\end{figure}

\vspace{2.5em}
\begin{figure*}[t]
\hrule
    \centering
    \begin{minipage}{0.55\textwidth}
    \tiny
    \emph{A \textcolor{darkgreen}{\textbf{\emph{black Labrador in a snug dark leather jacket}}} sits in a small,
    round, worn \textcolor{darkgreen}{\textbf{\emph{wooden coracle}}} on \textcolor{seafoam}{\textbf{\emph{green-brown water}}}. 
    Two \textcolor{darkgreen}{\textbf{\emph{matte metal mugs}}} lie on the coracle’s boards, and the dog \textcolor{orange}{\textbf{\emph{picks up}}} one to sip.
    The Labrador turns its head to one side, scanning the surroundings with a \textcolor{darkgreen}{\textbf{\emph{worried gaze}}}. Its wet fur shines;
    the jacket shows a soft leather sheen; the wood grain is rough and damp.
    \textcolor{seafoam}{\textbf{\emph{Ripples spread out in all directions}}} from the coracle and touch the faint
    reflections. The boat and the dog both mirror on the surface. Not far ahead,
    \textcolor{darkgreen}{\textbf{\emph{low green banks come into view}}}. The shore’s plants are barely visible, with dim
    reflections on the water. \textcolor{seafoam}{\textbf{\emph{Overcast sky gives soft, even light}}}; colors remain
    natural and muted. The scene stays calm yet slightly tense, with only the dog
    and boat—no people present.} 
    \end{minipage}%
    \hfill
    \begin{minipage}{0.43\textwidth}
    \includegraphics[width=0.95\linewidth]{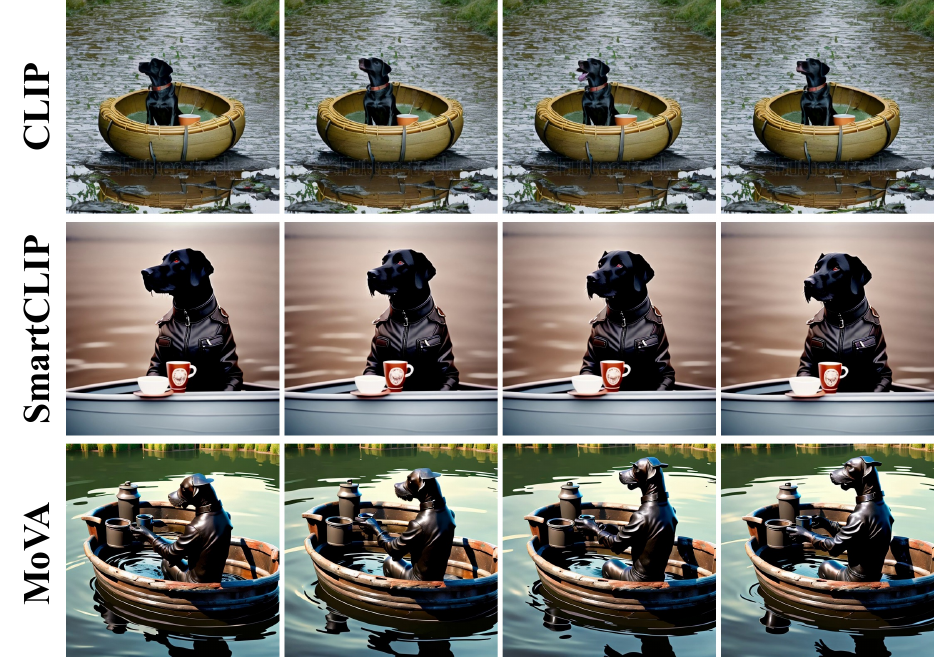}
    \end{minipage}
    \hrule
    \vspace{0.8em}
    \caption{
        \textbf{Example of Long-text-to-video generation.} 
        We replace the CLIP text encoder in VideoCrafter2~\cite{chen2024videocrafter2}. 
        Compared with existing text encoders trained through image–text alignment when used for text-to-video generation, our method MoVA comprehensively captures modular temporal-level and object-level concepts.
        MoVA can generate details such as the \emph{matte metal mug} and the \emph{pick up} motion. 
    }
    \label{fig:example_t2v}
\end{figure*}

\vspace{-0.8em}
\paragraph{Long-text-to-video generation.}
Figure~\ref{fig:example_t2v} highlights MoVA's ability to keep long textual descriptions intact when serving as the language interface for text-to-video generators such as VideoCrafter2~\cite{chen2024videocrafter2}. 
Compared with CLIP and SmartCLIP—which focus exclusively on image–text alignment—our method generates videos that are more vivid and more faithfully aligned with the text.
TMN dynamically selects span-level concepts and CMN aligns them with temporally evolving frames, which allows the generator to render subtle objects (e.g., the \emph{matte metal mug}) and actions (the dog \emph{picking up} the mug) that would otherwise vanish. For quantitative evaluation, we use VBench~\cite{huang2024vbench} to benchmark VideoCrafter2 in a zero-shot setting by replacing only its CLIP text encoder, and evaluate on 2,048 prompts randomly sampled from VidProM~\cite{wang2024vidprom}. The results are reported in Table~\ref{tab:comparison_results}.

\begin{table*}[tbp]
\centering
\footnotesize
{
\caption{\textbf{Text-to-video retrieval performance on various datasets.} ``$\uparrow$'' denotes that higher is better. ``$\downarrow$'' denotes that lower is better.}
\label{tab:retrieval_results_t2v} 
\resizebox{0.95\textwidth}{!}{%
\begin{tabular}{p{100pt}|ccccc|ccccc}
\toprule[1.5pt]
\multirow{2}{*}{\textbf{Method}} & \multicolumn{5}{c|}{\textbf{MSVD}} & \multicolumn{5}{c}{\textbf{DiDeMo}} \\ 
\cmidrule(rl){2-6}\cmidrule(rl){7-11}
  & R@1$\uparrow$ & R@5$\uparrow$ & R@10$\uparrow$ & MdR$\downarrow$ & MnR$\downarrow$
& R@1$\uparrow$ & R@5$\uparrow$ & R@10$\uparrow$ & MdR$\downarrow$ & MnR$\downarrow$    \\ \midrule
CLIP4Clip~\cite{luo2022clip4clip} & 47.4 & 77.8 & 85.6 & 2.0 & 10.3 & 44.8 & 73.4 & 81.6 & 2.0 & 13.5 \\
DRL~\cite{wang2022disentangled} & 49.8 & 81.2 & 89.5 & 2.0 & 9.5 & 49.0 & 76.5 & 84.5 & 2.0 & 12.0 \\
X-CLIP~\cite{ma2022x} & 50.4 & 80.6 & 89.8 & \Frst{1.0} & 8.4 & 47.8 & 79.4 & 82.3 & 2.0 & 12.5 \\
ProST~\cite{Li2023ProgressiveSP} & 46.4 & 74.4 & 83.8 & 2.0 & 12.1 & 47.5 & 75.2 & 84.6 & 2.0 & 12.3 \\
InternVideo~\cite{wang2023internvid} & 44.2 & 74.5 & 84.1 & 2.0 & 10.9 & 50.8 & 78.8 & 86.6 & \Frst{1.0} & 6.9 \\
VideoCLIP-XL~\cite{wang2024videoclip} & 48.6 & 81.0 & 86.6 & \Frst{1.0} & 10.2 & 38.6 & 63.6 & 73.0 & 3.0 & 49.2 \\
\midrule
\rowcolor{aliceblue!60} \textbf{MoVA (Ours)} & \Frst{52.6} & \Frst{83.0} & \Frst{90.6} & \Frst{1.0} & \Frst{7.8} & \Frst{57.5} & \Frst{83.2} & \Frst{91.6} & \Frst{1.0} & \Frst{4.8} \\  
\midrule[1.25pt]
\multirow{2}{*}{\textbf{Method}} & \multicolumn{5}{c|}{\textbf{VideoUFO}} & \multicolumn{5}{c}{\textbf{UltraVideo}} \\ 
\cmidrule(rl){2-6}\cmidrule(rl){7-11}
  & R@1$\uparrow$ & R@5$\uparrow$ & R@10$\uparrow$ & MdR$\downarrow$ & MnR$\downarrow$
& R@1$\uparrow$ & R@5$\uparrow$ & R@10$\uparrow$ & MdR$\downarrow$ & MnR$\downarrow$    \\ \midrule
CLIP4Clip~\cite{luo2022clip4clip} & 34.5 & 62.1 & 72.4 & 3.0 & 45.1 & 43.3 & 74.6 & 84.9 & 2.0 & 9.2 \\
DRL~\cite{wang2022disentangled} & 24.2 & 45.7 & 54.9 & 7.0 & 227.5 & 35.9 & 65.9 & 76.5 & 3.0 & 16.6 \\
X-CLIP~\cite{ma2022x} & 37.7 & 65.6 & 75.6 & 2.0 & 35.4 & 41.9 & 73.5 & 84.3 & 2.0 & 7.7 \\
ProST~\cite{Li2023ProgressiveSP} & 48.2 & 74.5 & 82.7 & 2.0 & 28.0 & 51.8 & 83.5 & 90.3 & \Frst{1.0} & 5.7 \\
InternVideo~\cite{wang2023internvid} & 42.8 & 70.2 & 76.4 & 2.0 & 33.6 & 35.4 & 65.9 & 76.3 & 3.0 & 14.0 \\
VideoCLIP-XL~\cite{wang2024videoclip} & 57.4 & 82.2 & 88.4 & \Frst{1.0} & 23.6 & 42.6 & 71.8 & 82.2 & 2.0 & 8.4 \\
\midrule
\rowcolor{aliceblue!60} \textbf{MoVA (Ours)} & \Frst{62.4} & \Frst{87.0} & \Frst{92.3} & \Frst{1.0} & \Frst{6.9} & \Frst{58.5} & \Frst{87.8} & \Frst{94.4} & \Frst{1.0} & \Frst{3.4} \\ 
\bottomrule[1.5pt]
\end{tabular}
}}

\end{table*}


\begin{table}[t]
  \centering
 \caption{Quantitative evaluation results on VBench.} 
\label{tab:comparison_results}
  
  \small
  \setlength{\tabcolsep}{8pt}
  \renewcommand{\arraystretch}{1.2}
  
  \resizebox{0.95\textwidth}{!}{%
    \begin{tabular}{l|ccccc|c}
    \toprule
    \textbf{Method} & 
    \begin{tabular}[c]{@{}c@{}}Subject\\ consistency\end{tabular} & 
    \begin{tabular}[c]{@{}c@{}}Background\\ consistency\end{tabular} & 
    \begin{tabular}[c]{@{}c@{}}Motion\\ smoothness\end{tabular} & 
    \begin{tabular}[c]{@{}c@{}}Aesthetic\\ quality\end{tabular} & 
    \begin{tabular}[c]{@{}c@{}}Imaging\\ quality\end{tabular} & 
    Avg. \\
    \midrule
    CLIP & 0.9700 & 0.9690 & 0.9799 & 0.6271 & 0.6822 & 0.8456 \\
    SmartCLIP & 0.9671 & 0.9722 & \textbf{0.9856} & 0.7064 & 0.6980 & 0.8659 \\
    MoVA (Ours) & \textbf{0.9776} & \textbf{0.9766} & 0.9849 & \textbf{0.7126} & \textbf{0.7023} & \textbf{0.8708} \\
\bottomrule[1.5pt]
    \end{tabular}
  }
\end{table}

\begin{figure}[t]
  \centering
   \includegraphics[width=0.7\linewidth]{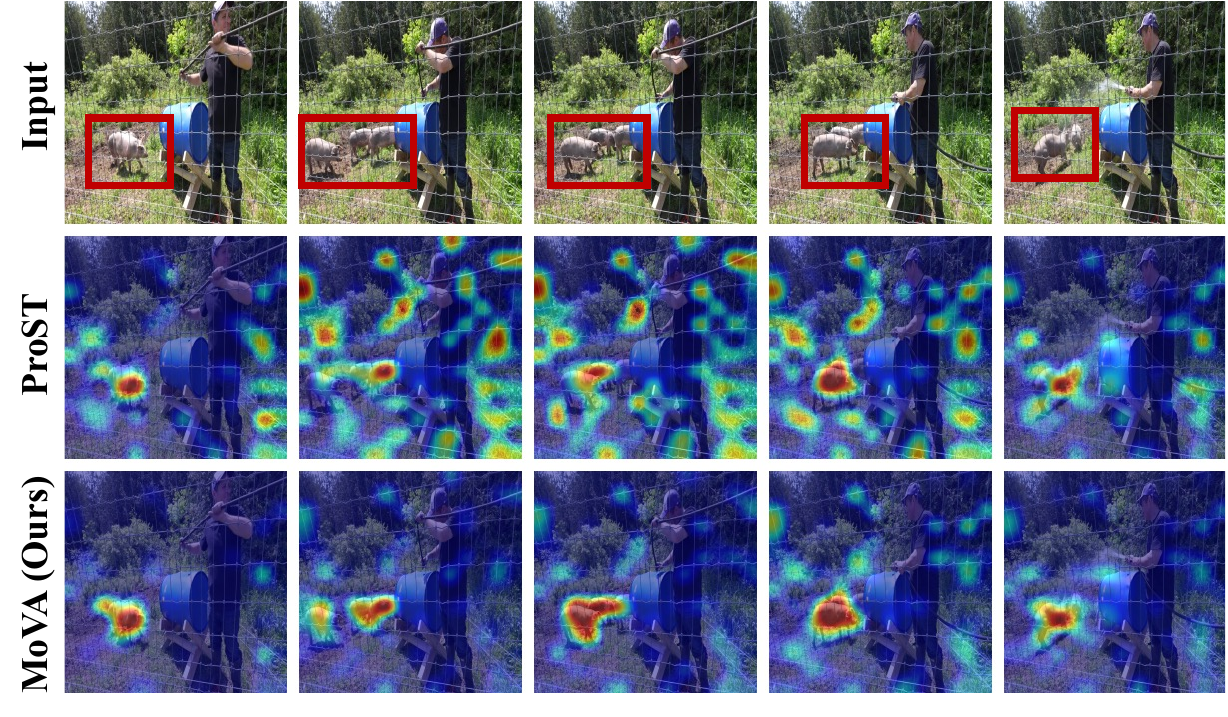}
   \caption{Visualization of learned concepts.
We perform representation visualization by formulating a proxy classification task. The cosine similarity between frame embeddings across time and a generated caption (e.g., ``some pigs'') serves as the classification score for GradCAM~\cite{Selvaraju2016GradCAMVE} attribution. Red boxes mark the temporally active regions. }
   \label{fig:concept_vis}
\end{figure}

\paragraph{Temporal mask analysis.}
Figure~\ref{fig:top3words} visualizes the top-$3$ textual spans selected by TMN across an ActivityNet video. It shows our model can select the aligned portion of the global text for various frames. 
This provides a perspective on addressing temporal misalignment, and it avoids using identical masks for adjacent frames, which reflects the role of the per-frame contrastive loss in Eq.~\eqref{eq:frame_loss}.

\paragraph{Concept-level visualization.}
Concept GradCAMs in Figure~\ref{fig:concept_vis} reveal that CMN isolates object-centric and motion-centric bases: when the caption emphasizes ``some pigs,'' activations gather around the animals despite distracting background motion. This supports our claim that semantic asymmetry must be modeled bidirectionally—text modules decide which frames to attend to, and frame-conditioned visual masks decide which text concepts remain active. Together with the retrieval gains, these qualitative trends provide high-level evidence that MoVA preserves global semantics while remaining compositional.

\subsection{Ablation Study}
\begin{figure*}[t]
  \centering
   \includegraphics[width=\linewidth]{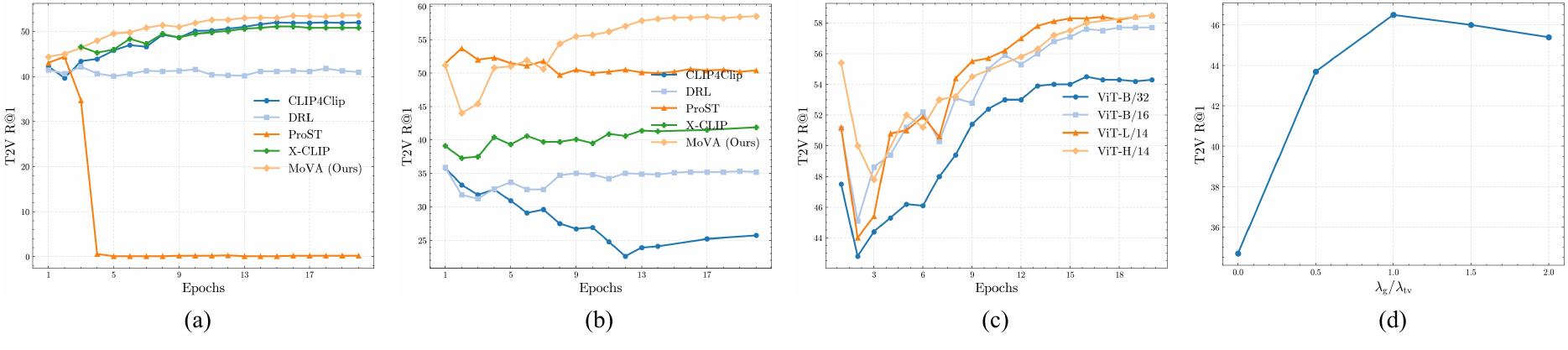}
   \caption{(a) Scaling with epochs on ActivityNet;
	(b) Scaling with epochs on UltraVideo;
	(c) Scaling trends across ViT backbones;
	(d) Ablation on loss-coefficient ratio $\lambda_g/\lambda_{tv}$.  }
   \label{fig:ablation}
\end{figure*}

\paragraph{Scalability.}
Panels (a) and (b) of Figure~\ref{fig:ablation} compare video-text retrieval when training on ActivityNet and UltraVideo. 
 MoVA exhibits steadily improving retrieval performance as training epochs increase.
Panel (c) further shows consistent improvements when swapping ViT-B/16, ViT-L/14, and ViT-H/14 backbones~\cite{radford2021learning}, indicating that TMN/CMN act as architecture-agnostic plugs rather than overfitting to a specific capacity regime. Collectively, these curves demonstrate that the dual asymmetric projections scale gracefully with both data volume and model size, which is essential for the long-form scenarios highlighted in the Introduction.

\paragraph{Loss weighting.}
Figure~\ref{fig:ablation}(d) studies the interaction between the global retrieval loss $\Lr$ and the modular alignment loss $\Ltv,\Lvt$. We fix $\lambda_{\mathrm{tv}}=\lambda_{\mathrm{vt}}=0.5$ and sweep $\lambda_{\mathrm{g}}$, plotting performance against $\lambda_{\mathrm{g}}/\lambda_{\mathrm{tv}}$. Extremely small ratios (e.g., $0.2$) under-emphasize the global constraint and slightly hurt MnR, whereas overly large ratios ($>1.2$) collapse the per-frame masks, validating the sparsity/coverage trade-off. The sweet spot around $1.0$ balances whole-video grounding with frame-level disentanglement.

\begin{wraptable}{r}{0.53\textwidth}
\vspace{-2.8em}
\centering
\caption{\textbf{Ablation studies on the ActivityNet dataset.} ``$\uparrow$'' denotes that higher is better. ``$\downarrow$'' denotes that lower is better.}
\label{tab:ablation}
\small
\setlength{\tabcolsep}{3.5pt}
\resizebox{0.50\textwidth}{!}{
\begin{tabular}{c|cccc|cccc}
\toprule[1.2pt]
\multirow{2}{*}{\begin{tabular}[c]{@{}c@{}}Ablation\\ Items\end{tabular}} &
\multicolumn{4}{c|}{Text $\rightarrow$ Video} &
\multicolumn{4}{c}{Video $\rightarrow$ Text} \\
\cmidrule(lr){2-5}\cmidrule(lr){6-9}
& R@1$\uparrow$ & R@5$\uparrow$ & R@10$\uparrow$ & MnR$\downarrow$
& R@1$\uparrow$ & R@5$\uparrow$ & R@10$\uparrow$ & MnR$\downarrow$ \\
\midrule
i   & 43.7 & 73.5 & 84.0 & 7.9  & 42.5 & 72.9 & 84.2 & 7.7 \\
ii  & 35.6 & 65.1 & 77.6 & 11.7 & 33.7 & 62.4 & 75.8 & 11.6 \\
iii & 42.0 & 70.0 & 81.5 & 9.7  & 41.9 & 70.8 & 80.8 & 9.0 \\
iv  & \textbf{47.6} & \textbf{77.4} & \textbf{87.0} & \textbf{7.0}
    & \textbf{46.7} & \textbf{76.6} & \textbf{86.8} & \textbf{6.3} \\
\bottomrule[1.2pt]
\end{tabular}}
\vspace{-2.0em}
\end{wraptable}
\paragraph{Effectiveness of dual asymmetric projections.}
Table~\ref{tab:ablation} summarizes three intermediate variants derived from the SmartCLIP initialization. (i) Fine-tuning with only the global retrieval loss $\Lr$ removes all modular objectives, corresponding to the image-text setting and yielding the weakest R@1. (ii) We retain SmartCLIP's single-direction visual$\rightarrow$text masking by learning a 3D mask from video frames to text tokens; this ignores the text-side global guidance emphasized in Section~\ref{sec:intro} and is unable to resolve temporal misalignment, leading to the sharpest degradation (R@1 $35.6$). (iii) We keep our TMN but apply temporal masks directly on raw text tokens rather than on the global text representation. (iv) The full MoVA stacks both TMN and CMN so that text can guide frame selection while frames can reweight textual concepts. These targeted ablations corroborate that learning dual asymmetric projections is effective for video-text alignment.

\section{Conclusion}
\label{sec:conclusion}
In this paper, we address temporal misalignment and semantic asymmetry in video--text alignment by formulating identification conditions that explicitly connect textual descriptions to their atomic visual counterparts. Building on these insights, MoVA introduces dual asymmetric projections that maintain global semantics while isolating frame-specific concepts, leading to disentangled and compositional cross-modal representations. Quantitative evaluations and qualitative studies jointly validate the theoretical claims, confirming that principled structure can translate into practical gains for multimodal video-text alignment. Looking ahead, MoVA will be integrated with next-generation multimodal foundation models to better handle long-horizon narratives, event compositionality, and fine-grained temporal grounding.

\section*{Acknowledgments}
We would like to acknowledge the support from NSF Award No.~2229881, AI Institute for Societal Decision Making (AI-SDM), the National Institutes of Health (NIH) under Contract R01HL159805, and grants from Quris AI, Florin Court Capital, MBZUAI-WIS Joint Program, and the Al Deira Causal Education project.

\bibliographystyle{splncs04}
\bibliography{main}

\clearpage
\section*{Supplementary Material}
\appendix
\setcounter{figure}{7}
\setcounter{table}{4}
\setcounter{equation}{9}

\section{Proof}
\label{sec:proof}
Our identification analysis shows that the selective, mask-based alignment objective in (\ref*{eq:contrastive_objective_mova}) recovers caption-conditioned semantic blocks up to blockwise invertible maps and extends to unions/intersections of such blocks. We study the constrained limit of the loss (alignment tolerance $\eta\!\to\!0$), which enforces exact masked equality at optimum
\[
\hzV_t\odot \hat{\mm}^{\mathrm V}_t \;=\; \hzT\odot \hat{\mm}^{\mathrm T}_t,\qquad t=1,\ldots,T,
\]
mirroring the dual-mask generative relation (\ref*{eq:vt_constraint}). Because similarities are computed \emph{after masking}, the loss constrains only the active coordinates; the sparsity term selects minimal supports that still achieve strong positive alignment and negative separation.

\noindent In this setting, Lemma~\ref{lem:vt_step1_projection} proves that masked alignment forces nuisance invariance on the active coordinates and yields blockwise invertible reparameterizations of the text and visual blocks, reducing identification to choosing the correct coordinates. Leveraging mask recurrence and view diversity, Lemma~\ref{lem:vt_step2_mask} shows that the learned text masks are class-wise consistent and support-correct at the optimum (no extras, no misses), a consequence of the alignment geometry plus the sparsity/contrastive trade-off. With these blocks fixed, Theorem~\ref{thm:vt_step3_blocks} identifies any union or intersection of recurring blocks up to blockwise invertible maps, and the corresponding claims hold on the visual side via the dual-mask equalities.

\begin{lemma}[Nuisance invariance on the active block and blockwise invertibility]
\label{lem:vt_step1_projection}
Let $\rtt:=\hgT\!\circ\!\gT$ and $\rii:=\hgV\!\circ\!\gV$.
Fix a frame $t$ and assume Condition~\ref*{cond:identification_mova}(i) (smoothness and local invertibility of the modality generators and encoders).
Assume further that in a small neighborhood of the data point the learned mask supports are stable, i.e., ${\rm supp}(\hat\mT_t)$ and ${\rm supp}(\hat\mV_t)$ do not change.
Define the binary diagonal projectors
\[
P_T:=\mathrm{Diag}(\hat\mT_t),\qquad P_V:=\mathrm{Diag}(\hat\mV_t).
\]
At an optimum of the selective-alignment objective (the constrained limit of (\ref*{eq:contrastive_objective_mova})), the masked alignment holds
\begin{equation}
\label{eq:vt_step1_align}
P_T\,\hzT \;=\; P_V\,\hzV_t,\quad
\hzT=\rtt(\zT,\veT),\quad \hzV_t=\rii(\zV_t,\veV_t).
\end{equation}
Then:
\begin{enumerate}[leftmargin=1.2em,itemsep=2pt,topsep=4pt]
\item (\emph{Nuisance invariance on the text active block}) For every coordinate $i$ with $[P_T]_{ii}=1$ and every text-nuisance coordinate $j$,
\(
\frac{\partial[\hzT]_i}{\partial[\veT]_j}=0
\).
Equivalently, $P_T\,\hzT$ is independent of $\veT$.
\item (\emph{Blockwise invertible reparameterizations}) There exist smooth, locally invertible maps $h_{\mathrm T}$ and $h_{\mathrm V}$ (defined on the respective active blocks) such that
\[
P_T\,\hzT \;=\; P_T\,h_{\mathrm T}(\zT),\qquad
P_V\,\hzV_t \;=\; P_V\,h_{\mathrm V}(\zV_t).
\]
\end{enumerate}
\end{lemma}

\begin{proof}
\paragraph{Step A: Fix projectors and use alignment.}
Because mask supports are locally constant, $P_T$ and $P_V$ are fixed binary projectors in the neighborhood.
By optimality of the selective-alignment loss, \eqref{eq:vt_step1_align} holds.
Importantly, the right-hand side $P_V\,\hzV_t=P_V\,\rii(\zV_t,\veV_t)$ does not contain the text nuisance $\veT$.

\paragraph{Step B: Nuisance invariance on the active block (contradiction via partial derivatives).}
Fix any active coordinate $i$ with $[P_T]_{ii}=1$ and any nuisance coordinate $j$ of $\veT$.
Suppose, towards a contradiction, that
\(
\partial [\rtt(\zT,\veT)]_i/\partial [\veT]_j\neq 0
\)
at some $(\zT_0,\veT_0)$.
By smoothness, this partial derivative keeps a nonzero sign in a small interval of $[\veT]_j$ around $[\veT_0]_j$.
Holding $(\zV_t,\veV_t)$ fixed and varying only $[\veT]_j$ along that interval makes the left-hand side component $[P_T\,\rtt(\zT,\veT)]_i$ change strictly, while the right-hand side component $[P_V\,\rii(\zV_t,\veV_t)]_i$ remains constant (it does not depend on $\veT$), contradicting \eqref{eq:vt_step1_align}.
Hence $\partial[\hzT]_i/\partial[\veT]_j=0$ for all active $i$ and all $j$, i.e., $P_T\,\hzT$ is invariant to $\veT$.

\paragraph{Step C: Blockwise invertible reparameterizations (local diffeomorphisms on the active block).}
Define $\phi_T(\zT):=P_T\,\rtt(\zT,\veT_0)$ for any fixed $\veT_0$; this is well-defined because $P_T\,\hzT$ is independent of $\veT$ by Step~B.
By Condition~\ref*{cond:identification_mova}(i), the composite encoder $\rtt$ is smooth and locally invertible with respect to the semantic coordinates, and its restriction to the active coordinates selected by $P_T$ has full (block) rank.
Therefore, by the inverse function theorem (or constant-rank theorem), $\phi_T$ defines a local diffeomorphism between the true text latent block and its encoded image on that block.
We denote the induced reparameterization by $h_{\mathrm T}$ and obtain $P_T\,\hzT=P_T\,h_{\mathrm T}(\zT)$.
The same argument on the visual side yields $P_V\,\hzV_t=P_V\,h_{\mathrm V}(\zV_t)$.
\end{proof}

\vspace{0.25em}

\begin{assumption}[Mask Recurrence (class-based)]
\label{assump:mask_recurrence_step2}
There is a finite family of text-side masks $\{\mT^{(k)}\}_{k\in\cK}$ with $\Pr(\mT=\mT^{(k)})>0$.
For each $k$, the set $\cS_k:=\{t:\mT_t=\mT^{(k)}\}$ contains multiple samples (not necessarily adjacent in time), and conditioned on $\mT_t=\mT^{(k)}$ the visual latents/nuisances vary in a neighborhood (view diversity).
\end{assumption}

\begin{lemma}[Consistency and support-correctness of learned text masks]
\label{lem:vt_step2_mask}
Let $\rtt:=\hgT\!\circ\!\gT$ and $\rii:=\hgV\!\circ\!\gV$ and adopt the notation of Lemma~\ref{lem:vt_step1_projection}.
Assume Condition~\ref*{cond:identification_mova}(i)--(iv) and Assumption~\ref{assump:mask_recurrence_step2}.
Fix a class $k$ and write $H_T:=h_{\mathrm T}(\zT)$ and $H_{V,t}:=h_{\mathrm V}(\zV_t)$ from Lemma~\ref{lem:vt_step1_projection}.
At an optimum of the selective-alignment objective (\ref*{eq:contrastive_objective_mova}), the following hold:

\begin{enumerate}[leftmargin=1.2em,itemsep=0pt,topsep=4pt]
\item \textbf{Class-wise consistency.} For all $t_1,t_2\in\cS_k$,
\[
\hat{\mm}^{\mathrm T}_{t_1}=\hat{\mm}^{\mathrm T}_{t_2}\,.
\]

\item \textbf{Support-correctness.} Let $\mT^{(k)}$ be the true text mask of class $k$.
Then $\cB(\hat{\mm}^{\mathrm T}_{t})=\cB(\mT^{(k)})$ for all $t\in\cS_k$ (up to a block permutation).
\end{enumerate}
\end{lemma}

\begin{proof}
\textbf{Setup.}
By Lemma \ref{lem:vt_step1_projection}, on a support-stable neighborhood the masked alignment reads, for each $t$,
\begin{equation}\label{eq:masked_alignment_step2}
H_T\odot \hat{\mm}^{\mathrm T}_t \;=\; H_{V,t}\odot \hat{\mm}^{\mathrm V}_t,
\end{equation}
where $H_T=h_{\mathrm T}(\zT)$ depends only on $\zT$ (no $\veT$), while $H_{V,t}=h_{\mathrm V}(\zV_t)$ can vary across $t\in\cS_k$ by view diversity.

\medskip
\noindent\textbf{(1) Class-wise consistency.}
Fix $k$ and suppose, by contradiction, there exist $t_1,t_2\in\cS_k$ with $\hat{\mm}^{\mathrm T}_{t_1}\neq \hat{\mm}^{\mathrm T}_{t_2}$.
Let $\Delta:=\cB(\hat{\mm}^{\mathrm T}_{t_1})\triangle \cB(\hat{\mm}^{\mathrm T}_{t_2})$ be the symmetric difference.
Pick any $i\in\Delta$; without loss of generality assume $i\in\cB(\hat{\mm}^{\mathrm T}_{t_1})$ and $i\notin\cB(\hat{\mm}^{\mathrm T}_{t_2})$.
Taking coordinate $i$ in \eqref{eq:masked_alignment_step2}, we get
\[
[H_T]_i \;=\; [H_{V,t_1}]_i\,[\hat{\mm}^{\mathrm V}_{t_1}]_i, 
\qquad
0 \;=\; [H_{V,t_2}]_i\,[\hat{\mm}^{\mathrm V}_{t_2}]_i.
\]
By Assumption~\ref{assump:mask_recurrence_step2} (view diversity in the class), the right factors $[H_{V,t}]_i$ vary in a neighborhood across $t\in\cS_k$.
Thus, to satisfy the second equality for (almost) all such $t_2$, the only robust possibility is that $[\hat{\mm}^{\mathrm V}_{t_2}]_i=0$ on a full neighborhood; in particular the right-hand side is forced to zero generically.
But the first equality requires reproducing the \emph{same} nonzero value $[H_T]_i$ (which is independent of $t$) on some $t_1$.
Because $(H_{V,t})_{t\in\cS_k}$ explore a neighborhood, the pair of requirements is generically incompatible unless $i$ is \emph{never} selected on the text side within the class.
Consequently $\Delta$ must be empty.
Hence $\hat{\mm}^{\mathrm T}_{t_1}=\hat{\mm}^{\mathrm T}_{t_2}$ for all $t_1,t_2\in\cS_k$.

\medskip
\noindent\textbf{(2) Support-correctness (no extras, no misses).}
Write $\hat{\mm}^{\mathrm T}$ for the common class-wise text mask established above.

\emph{No extra coordinates.}
Suppose $\cB(\hat{\mm}^{\mathrm T})$ strictly contains $\cB(\mT^{(k)})$.
Let $E:=\cB(\hat{\mm}^{\mathrm T})\setminus \cB(\mT^{(k)})$ be the set of spurious text coordinates.
For any $i\in E$, the text side reveals $[H_T]_i$ in the positive pairs for class $k$, but the true cross-modal semantics in class $k$ do not require $i$.
In contrastive training, these exposed coordinates $i\in E$ (a) increase similarity to negatives that coincidentally activate $i$, and (b) are penalized by the sparsity term $\sum_t\|\hat{\mm}^{\mathrm T}_t\|_0$.
Turning them off strictly improves the objective without harming the positive alignment (since class-$k$ semantics do not need them).
Therefore the optimum cannot have $E\neq\emptyset$.

\emph{No missing coordinates.}
Suppose $\cB(\hat{\mm}^{\mathrm T})$ is a strict subset of $\cB(\mT^{(k)})$.
Let $M:=\cB(\mT^{(k)})\setminus \cB(\hat{\mm}^{\mathrm T})$ be the set of missing true coordinates.
For $i\in M$, the positive alignment in \eqref{eq:masked_alignment_step2} discards genuinely informative text dimensions $[H_T]_i$, reducing positive similarity and making hard negatives harder to separate.
Activating $i$ increases positive alignment while the sparsity cost grows only additively by $|M|$.
By the usual margin trade-off in the contrastive objective, the optimum cannot omit any $i\in \cB(\mT^{(k)})$.
Thus $M=\emptyset$.

Combining the two parts gives $\cB(\hat{\mm}^{\mathrm T})=\cB(\mT^{(k)})$ (up to a block permutation).
\end{proof}

\begin{theorem}[Block-wise identifiability for unions and intersections]
\label{thm:vt_step3_blocks}
Let $\{\mT^{(k)}\}_{k\in\cK}$ be the (true) text-side masks for the recurring classes, and let 
$\hat\mT^{(k)}$ be the corresponding learned masks established in Lemma~\ref{lem:vt_step2_mask} (class-wise consistent and support-correct, up to a block permutation).
For any finite subfamily $\cV\subseteq\cK$, define the union and intersection index sets
\[
\tilde{\cB}_{\cup}\ :=\ \bigcup_{k\in\cV}\cB(\mT^{(k)}),
\qquad
\tilde{\cB}_{\cap}\ :=\ \bigcap_{k\in\cV}\cB(\mT^{(k)}).
\]
Then the true latent sub-vectors $[\zz]_{\tilde{\cB}_{\cup}}$ and $[\zz]_{\tilde{\cB}_{\cap}}$ are identifiable
up to block-wise invertible maps; concretely, there exist smooth local reparameterizations
$h_{\cup},h_{\cap}$ such that for the text representation,
\[
[\hzT]_{\tilde{\cB}_{\cup}}=[h_{\cup}(\zT)]_{\tilde{\cB}_{\cup}},
\qquad
[\hzT]_{\tilde{\cB}_{\cap}}=[h_{\cap}(\zT)]_{\tilde{\cB}_{\cap}},
\]
and the same holds for the aligned visual blocks via the dual-mask equalities.
\end{theorem}

\begin{proof}
Let $H_T:=h_{\mathrm T}(\zT)$ and $H_{V,t}:=h_{\mathrm V}(\zV_t)$ be the block-wise reparameterizations from Lemma~\ref{lem:vt_step1_projection}.
By Lemma~\ref{lem:vt_step2_mask}, for each class $k$ there exists a learned text mask $\hat\mT^{(k)}$ whose support equals $\cB(\mT^{(k)})$ (up to a block permutation), and for any frame $t$ belonging to class $k$ we have the masked alignment
\begin{equation}\label{eq:step3_align}
\mathrm{Diag}(\hat\mT^{(k)})\,\hzT
\;=\;
\mathrm{Diag}(\hat\mV_t)\,\hzV_t
\;=\;
\mathrm{Diag}(\hat\mT^{(k)})\,H_T,
\end{equation}
where the last equality is Lemma~\ref{lem:vt_step1_projection}’s invariance on the text active block.

\paragraph{(A) Unions.}
Define the union mask (projector) $P_\cup:=\mathrm{Diag}(\1_{\tilde{\cB}_\cup})$.
Take any coordinate $i\in\tilde{\cB}_\cup$.
Then $i\in\cB(\mT^{(k)})$ for some $k\in\cV$, hence by \eqref{eq:step3_align} we have 
$[\hzT]_i=[H_T]_i$ (up to the fixed within-block permutation inherited from $\hat\mT^{(k)}$).
Since this holds for every $i\in\tilde{\cB}_\cup$,
\[
P_\cup\,\hzT \;=\; P_\cup\,H_T.
\]
By Lemma~\ref{lem:vt_step1_projection}, the restriction of $H_T$ to any active coordinates is a local diffeomorphism of the corresponding true latent block; in particular, the Jacobian of $\zT\mapsto [H_T]_{\tilde{\cB}_\cup}$ has full (block) rank in a neighborhood.
Therefore $[H_T]_{\tilde{\cB}_\cup}$ is an invertible (block-wise) reparameterization of $[\zT]_{\tilde{\cB}_\cup}$, which we denote by $[h_{\cup}(\zT)]_{\tilde{\cB}_\cup}$.
Hence $[\hzT]_{\tilde{\cB}_\cup}=[h_{\cup}(\zT)]_{\tilde{\cB}_\cup}$ and the union block is identifiable up to a block-wise invertible map.

\paragraph{(B) Intersections.}
For intersections, take any two classes $k_1,k_2\in\cV$ and consider the common index set 
$\cB^{(1,2)}:=\cB(\mT^{(k_1)})\cap\cB(\mT^{(k_2)})$.
Restricting \eqref{eq:step3_align} to $\cB^{(1,2)}$ for frames from class $k_1$ and from class $k_2$ yields
\begin{equation}
\begin{split}
    [H_T]_{\cB^{(1,2)}}
    &= [\hzT]_{\cB^{(1,2)}} \\
    &= \big[\mathrm{Diag}(\hat\mV_t)\,\hzV_t\big]_{\cB^{(1,2)}}
    \quad\text{for both }k_1\text{ and }k_2.
\end{split}
\end{equation}
Thus the same sub-vector $[H_T]_{\cB^{(1,2)}}$ is compelled by alignment across multiple classes/frames, pinning it to the text coordinates in the intersection.
By Lemma~\ref{lem:vt_step1_projection}, the restriction $\zT\mapsto[H_T]_{\cB^{(1,2)}}$ is locally invertible on that block, so we obtain identifiability of $[\zT]_{\cB^{(1,2)}}$ up to a block-wise invertible map.
Since finite intersections can be built iteratively (pairwise intersections are associative/commutative at the index-set level), this extends to $\tilde{\cB}_\cap=\bigcap_{k\in\cV}\cB(\mT^{(k)})$:
\[
[\hzT]_{\tilde{\cB}_\cap}=[H_T]_{\tilde{\cB}_\cap}=[h_{\cap}(\zT)]_{\tilde{\cB}_\cap}.
\]

\paragraph{(C) Visual side.}
For any block selected on the text side, \eqref{eq:step3_align} implies the same block is realized on the visual side via $\mathrm{Diag}(\hat\mV_t)\,\hzV_t$, so the corresponding visual sub-vectors are also identifiable up to block-wise invertible maps by Lemma~\ref{lem:vt_step1_projection}’s $h_{\mathrm V}$.

Combining (A)--(C) establishes the claim for both unions and intersections.
\end{proof}

\begin{figure*}[t]
    \hrule
    \vspace{0.1cm}
    \centering
    
    \begin{minipage}{0.55\textwidth}
    \small
    \emph{A highly stylized, \textcolor{darkgreen}{\textbf{\emph{slender meerkat}}} stands upright, serving as the central figure in a \textcolor{seafoam}{\textbf{\emph{peculiar, dimly lit environment}}}. She wears an elaborate \textcolor{darkgreen}{\textbf{\emph{long coat patchworked with tattered fabrics}}} in various colors, finished with a striking \textcolor{darkgreen}{\textbf{\emph{magenta collar and lapels}}}. The meerkat stands on a rough floor beneath a focused spotlight. Behind her, a backdrop of massive, abstract, vividly colored \textcolor{darkgreen}{\textbf{\emph{curtains}}}—yellow, blue, and burgundy—overlaid with \textcolor{darkgreen}{\textbf{\emph{stylized graffiti}}}, suggests a fantastical, tribal-inspired aesthetic. She \textcolor{orange}{\textbf{\emph{sways in place}}}, rocking from heel to toe with small shoulder rolls, tilting her head toward the light as she \textcolor{orange}{\textbf{\emph{delivers a clear, lilting melody}}}; the tune rises and falls in long, sustained phrases, tender and affecting.}
    \end{minipage}%
    \hfill
    \begin{minipage}{0.43\textwidth}
    \includegraphics[width=0.95\linewidth]{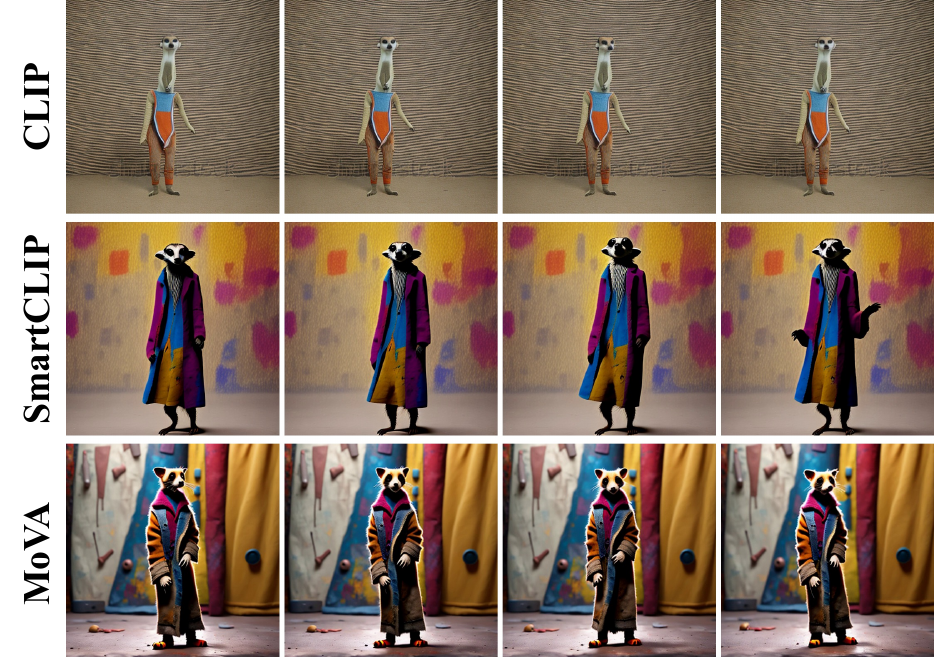}
    \end{minipage}
    
    \vspace{0.2cm}
    \hrule 
    \vspace{0.2cm}

    \begin{minipage}{0.55\textwidth}
    \small
    \emph{A highly detailed, \textcolor{darkgreen}{\textbf{\emph{anthropomorphic humanoid robot}}} on the right is engaged in \textcolor{orange}{\textbf{\emph{metalworking and assembly}}}. It wears a \textcolor{darkgreen}{\textbf{\emph{blue-black textured jacket}}} spattered with rust-colored orange and white paint or grime, giving it a worn, workmanlike aesthetic. Its head is enclosed in a \textcolor{darkgreen}{\textbf{\emph{metal welding mask}}}—heavily rusted, with complex wiring or hoses trailing from the back. The robot \textcolor{orange}{\textbf{\emph{grips}}} a \textcolor{darkgreen}{\textbf{\emph{glowing, bright-orange, intensely hot piece of material}}} while interacting with the angular metallic arm of a second, partially visible robot on the left. The scene is \textcolor{seafoam}{\textbf{\emph{industrial and gritty}}}, marked by \textcolor{darkgreen}{\textbf{\emph{heavy metal surfaces}}}, exposed wires, and a focus on mechanical labor.}
    \end{minipage}%
    \hfill
    \begin{minipage}{0.43\textwidth}
    \includegraphics[width=0.95\linewidth]{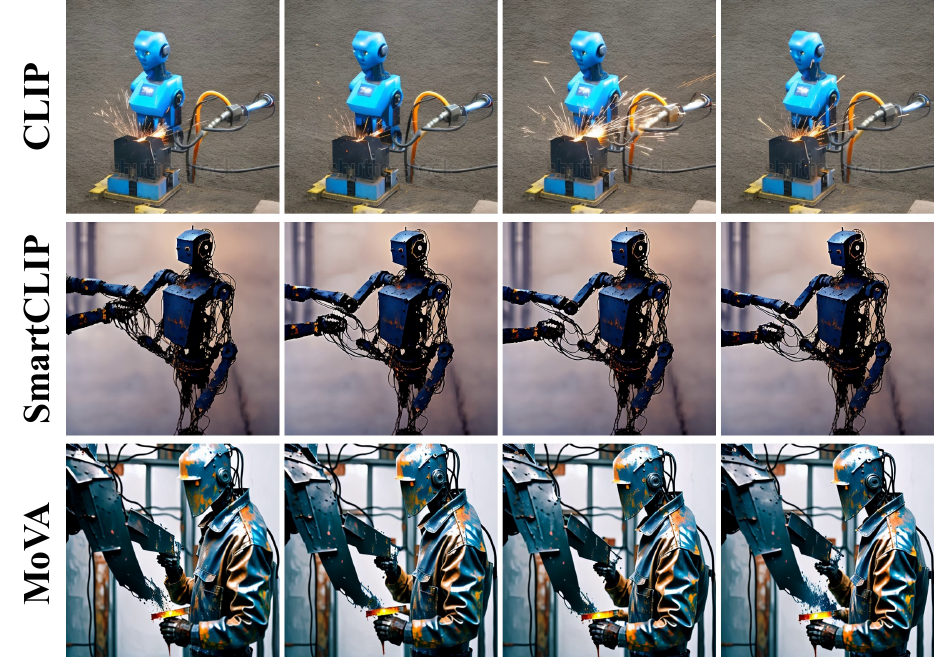}
    \end{minipage}

    \vspace{0.1cm}
    \hrule
    \vspace{0.1cm}
    
    \caption{
        \textbf{Additional Comparison Examples of Long-text-to-video generation.} 
        These examples further demonstrate MoVA's capability in handling complex descriptive prompts. 
        \textbf{Top:} The model accurately captures the stylized aesthetic, rendering the \emph{patchworked coat} and the specific motion of \emph{swaying} while singing. 
        \textbf{Bottom:} MoVA successfully generates the fine-grained details of the \emph{rusted welding mask} and the \emph{glowing material} within an \emph{industrial} atmosphere.
    }
    \label{fig:sup_qualitative} 
\end{figure*}

\section{Dataset Details}
We evaluate MoVA on five video-text retrieval datasets: ActivityNet \cite{krishna2017dense}, MSVD \cite{chen2011collecting}, DiDeMo \cite{anne2017localizing}, VideoUFO~\cite{wang2025videoufo}, and UltraVideo~\cite{xue2025ultravideo}.

\paragraph{ActivityNet}~\cite{krishna2017dense} contains $20{,}000$ YouTube videos and $100{,}000$ captions (849 hours in total) covering $200$ activity types, with an average length of $153$ seconds. Following \cite{luo2022clip4clip}, we concatenate the multiple descriptions per video into a single paragraph for paragraph-to-video retrieval. We train on $10{,}000$ videos and evaluate on the \texttt{val1} split with $5{,}000$ videos.

\paragraph{MSVD}~\cite{chen2011collecting} includes $1{,}970$ short clips ($1$--$62$ seconds) and about $40$ captions per video. We adopt the standard split of $1{,}200$/$100$/$670$ videos for train/val/test and report retrieval under the multi-caption evaluation protocol~\cite{luo2022clip4clip}.

\paragraph{DiDeMo}~\cite{anne2017localizing}  consists of $10{,}000$ videos and $40{,}000$ captions (average length $\sim30$ seconds). Similar to ActivityNet, we perform paragraph-to-video retrieval. The train/val/test sets contain $8{,}395$/$1{,}065$/$1{,}004$ videos, and results are reported on the test split.

\paragraph{VideoUFO}~\cite{wang2025videoufo} is a million-scale user-focused benchmark with $1.09$M clips across $1{,}291$ topics at $720$p resolution. Captions average $155.5$ words, while videos average $12.6$ seconds (about $3.5$K hours in total), posing retrieval challenges due to both scale and description richness. We randomly sample $95\%$ ($1{,}037{,}126$) clips for training and $5\%$ ($54{,}586$) for testing.

\paragraph{UltraVideo}~\cite{xue2025ultravideo} targets high-quality text-to-video research with 4K/8K footage and structured captions. It contains $17$K clips (143 hours) with $30.9$-second videos on average and very long descriptions (average $850.3$ words). We use the UltraVideo-Long subset and refer to it as UltraVideo for brevity, taking the first $75\%$ ($12{,}447$) clips from the original caption annotation file for training and the remaining $25\%$ ($4{,}150$) for testing.

Taken together, these datasets span short clips with concise captions to ultra-high-resolution videos paired with hundred-word paragraphs, offering a comprehensive testbed for scaling MoVA’s alignment ability across length, resolution, and topic breadth. VideoUFO and UltraVideo (introduced to video-text retrieval for the first time in this work) deliberately stress long, information-dense captions: short captions can retrieve many near-duplicate results, whereas long captions challenge the model to extract effective signals while avoiding loss of key concepts. The rich annotations also encourage multi-task reuse of the learned video–text representations, such as long-text-to-video generation.

\section{Additional Experiments}
\paragraph{Additional quantitative results.} Table~\ref{tab:retrieval_results_didemo} reports DiDeMo retrieval when every method uses a weaker ViT-B/32 backbone. MoVA still outperforms the strongest baseline (CLIP-VIP) by $+3.0$ R@1 and halves MnR (7.8 vs.\ 13.5) without leveraging auxiliary video–subtitle or video–caption corpora—we only fine-tune on the given benchmark. In Table~\ref*{tab:retrieval_results_t2v}, MoVA likewise surpasses high-resolution models such as VideoCLIP-XL~\cite{wang2024videoclip}, showing that the modular alignment remains effective without resorting to larger encoders.
\begin{table}[htbp]
\centering
\footnotesize
\caption{Text-to-video retrieval performance on the DiDeMo dataset (ViT-B/32 as backbone). ``$\uparrow$'' denotes that higher is better. ``$\downarrow$'' denotes that lower is better.}
\label{tab:retrieval_results_didemo}
\vspace{-0.7em}

\resizebox{0.5\linewidth}{!}{
\begin{tabular}{l|cccc}
\toprule[1.5pt]
\textbf{Method} & R@1$\uparrow$ & R@5$\uparrow$ & MdR$\downarrow$ & MnR$\downarrow$ \\ \midrule
CLIP4Clip~\cite{luo2022clip4clip} & 42.8 & 68.5 & 2.0 & 18.9 \\
X-CLIP~\cite{ma2022x} & 45.2 & 74.0 & 2.0 & 14.6 \\
CLIP-VIP~\cite{xue2022clip} & 47.4 & 75.2 & 2.0 & 13.5 \\
UCOFIA~\cite{wang2023unified} & 46.5 & 74.8 & 2.0 & 13.1 \\
ProST~\cite{Li2023ProgressiveSP} & 44.9 & 72.7 & 2.0 & 13.7 \\
\midrule
 \textbf{MoVA (Ours)} & \Frst{50.4} & \Frst{82.1} & \Frst{1.0} & \Frst{7.8} \\
\bottomrule[1.5pt]
\end{tabular}
}
\vspace{0.3em}
\end{table}

\begin{table}[htbp]
    \centering
    \footnotesize
    \caption{Quantitative comparison of total parameter counts (in Millions).}
    \label{tab:activitynet_params}
    \begin{tabular}{l|c}
        \toprule[1.5pt]
        \textbf{Method} & \textbf{Params (M)} \\
        \midrule
        CLIP~\cite{radford2021learning} & 149.6 \\
        SmartCLIP~\cite{Xie2025SmartCLIPMV} & 153.0 \\
        \midrule
        CLIP4Clip~\cite{luo2022clip4clip} & 162.3 \\
        DRL~\cite{wang2022disentangled} & 162.8 \\
        X-CLIP~\cite{ma2022x} & 162.8 \\
        ProST~\cite{Li2023ProgressiveSP} & 177.9 \\
        DiCoSA~\cite{Jin2023TextVideoRW} & 162.3 \\
        VideoCLIP-XL~\cite{wang2024videoclip} & 427.9 \\
         MoVA (Ours) & 174.7 \\
        \bottomrule[1.5pt]
    \end{tabular}
\end{table}

\paragraph{Parameters and training cost.} Table~\ref{tab:activitynet_params} reports total parameters. The Concept Mask Network adds $11{,}099{,}650$ ($\approx11.1$M) parameters and the Temporal Mask Network adds $7{,}946{,}753$ ($\approx7.9$M), bringing MoVA to $174.7$M—about a 7.6\% increase over CLIP4Clip ($162.3$M) yet far below VideoCLIP-XL ($427.9$M, from its released ViT-L/14 checkpoint). Compared with ProST ($177.9$M), MoVA uses roughly $1.8\%$ fewer parameters while delivering stronger retrieval and concept disentanglement. On an 8$\times$MI210 node, training ActivityNet for one epoch takes $76.7$ minutes for MoVA versus $92.9$ for CLIP4Clip, showing the added modules keep memory and compute overhead acceptable relative to the gains.

\paragraph{Ablations.} Tables~\ref{tab:sup_ablation} and~\ref{tab:sup_ablation_blocks} study frame counts and TMN block numbers. Performance is stable across frame budgets (32--160), peaking near 96 frames, indicating the model does not rely on blindly scaling frames. Increasing TMN blocks beyond a shallow depth does not yield further gains, so MoVA remains efficient without over-deep temporal masking. In addition, both $\Lsfdm$ and $\Ldfsm$ are necessary, as removing $\Ldfsm$ alone reduces ActivityNet R@1 by $3.6$. Furthermore, soft mask variant (same $L_1$ sparsity) is comparable (ActivityNet R@1/5/10 47.4/77.3/86.6; DiDeMo R@1 57.2; UltraVideo R@1 58.3). Hard masks remain default because of better interpretability.

\begin{table}[htbp]
\centering
\caption{\textbf{Ablation study on the number of input frames.} Evaluated on the ActivityNet dataset.}
\label{tab:sup_ablation} 
\vspace{-0.7em}
\small
\resizebox{0.6\textwidth}{!} 
{

\begin{tabular}{c|cccc|cccc} 
\toprule[1.2pt]
\multirow{2}{*}{{\begin{tabular}[c]{@{}c@{}}Frames\end{tabular}}} & \multicolumn{4}{c|}{Text $\rightarrow$ Video} & \multicolumn{4}{c}{Video $\rightarrow$ Text} \\ %
 & R@1$\uparrow$ & R@5$\uparrow$ & R@10$\uparrow$ & MnR$\downarrow$ & R@1$\uparrow$ & R@5$\uparrow$ & R@10$\uparrow$ & MnR$\downarrow$ \\ \hline

32 & 46.7 & 75.6 & 86.3 & 7.1 & 45.3 & 75.0 & 86.9 & 7.4 \\
64 & 47.2 & 76.4 & 86.2 & 6.6 & 45.1 & 75.4 & 87.1 & 6.6 \\
96 & 47.9 & 77.3 & 86.9 & 6.9 & 46.7 & 76.2 & 86.9 & 6.3 \\
128 & 47.8 & 77.4 & 87.0 & 7.0 & 46.7 & 76.6 & 86.8 & 6.3 \\
160 & 46.6 & 76.1 & 87.2 & 6.9 & 45.2 & 76.1 & 86.9 & 6.4 \\
\bottomrule[1.2pt]
\end{tabular}
}

\end{table}

\begin{table}[htbp]
\centering
\caption{\textbf{Ablation study on the number of TMN blocks.} Evaluated on the ActivityNet dataset. The Temporal Mask Network depth is varied while keeping other settings fixed.}
\label{tab:sup_ablation_blocks} 
\vspace{-0.7em}
\small
\resizebox{0.6\textwidth}{!} 
{

\begin{tabular}{c|cccc|cccc} 
\toprule[1.2pt]
\multirow{2}{*}{{\begin{tabular}[c]{@{}c@{}}Blocks\end{tabular}}} & \multicolumn{4}{c|}{Text $\rightarrow$ Video} & \multicolumn{4}{c}{Video $\rightarrow$ Text} \\ %
 & R@1$\uparrow$ & R@5$\uparrow$ & R@10$\uparrow$ & MnR$\downarrow$ & R@1$\uparrow$ & R@5$\uparrow$ & R@10$\uparrow$ & MnR$\downarrow$ \\ \hline

1 & 47.8 & 77.0 & 87.0 & 6.7 & 46.8 & 77.0 & 87.6 & 6.2 \\
2 & 47.8 & 77.4 & 87.0 & 7.0 & 46.7 & 76.6 & 86.8 & 6.3 \\
3 & 46.6 & 76.2 & 86.4 & 6.5 & 45.3 & 76.2 & 87.0 & 6.5 \\
4 & 47.8 & 76.9 & 87.1 & 6.3 & 46.8 & 76.8 & 87.1 & 6.2 \\
8 & 47.0 & 76.6 & 87.0 & 6.1 & 45.7 & 76.6 & 87.5 & 5.9 \\
\bottomrule[1.2pt]
\end{tabular}
}

\end{table}

\paragraph{More visualization examples.} Figure~\ref{fig:sup_qualitative} shows two long-prompt generations. The first captures a stylized meerkat in a dim, curtain-lined set with a patchworked coat, magenta collar, and subtle swaying while singing. The second depicts an industrial welding scene with a rusted mask, blue-black jacket, glowing orange material, and gritty metal surroundings. Beyond static fidelity, MoVA preserves temporal cues (e.g., rhythmic swaying, welding interaction) without drifting or collapsing motion, and avoids artifacts like the frame-level watermarks often seen when vanilla CLIP encoders are fine-tuned on video-text data. Similar to Figure~\ref*{fig:example_t2v}, we omit other video-text fine-tuned baselines (e.g., CLIP4Clip, DGL) because after the same number of iterations they lose structured text-to-video guidance, often producing near-identical outputs for different texts or misaligned videos.

\end{document}